\makeatletter
\let\MR\undefined
\makeatother

\documentclass[aos]{imsart}

\RequirePackage{amsthm,amsmath,amsfonts,amssymb}
\RequirePackage[numbers,sort&compress]{natbib}
\RequirePackage[colorlinks,citecolor=blue,urlcolor=blue]{hyperref}
\RequirePackage{graphicx}

\usepackage[T1]{fontenc}
\usepackage{blindtext}
\usepackage{amsmath}
\usepackage{amssymb}
\usepackage{graphicx,psfrag,epsf}
\usepackage{enumerate,verbatim}
\usepackage{enumitem}
\usepackage{mathalfa}
\usepackage{prodint}
\usepackage{multirow}
\usepackage{bm}
\usepackage[english]{babel}
\usepackage{makecell}
\usepackage{algorithm}
\usepackage{chngcntr}

\usepackage{tabularx}
\newcommand\setrow[1]{\gdef\rowmac{#1}#1\ignorespaces}
\newcommand\clearrow{\global\let\rowmac\relax}
\clearrow

\usepackage{algcompatible}
\algnewcommand\INPUT{\item[\textbf{Input:}]}%
\algnewcommand\OUTPUT{\item[\textbf{Output:}]}%
\usepackage{longtable,booktabs}
\usepackage{mathtools}
\usepackage{color, colortbl}
\definecolor{Gray}{gray}{0.9}
\definecolor{Gray2}{gray}{0.7}
\newcolumntype{g}{>{\columncolor{Gray}}c}
\newcolumntype{G}{>{\columncolor{Gray2}}c}
\usepackage{rotating}
\usepackage{framed,lipsum}
\usepackage[flushleft]{threeparttable}
\usepackage{float}

\usepackage[pagewise]{lineno}
\usepackage{framed,lipsum}

\usepackage{changebar} 

\def\wt{\widetilde}

\def\wh{\widehat}

\newcommand{\beq}{\begin{equation}}
	\newcommand{\eeq}{\end{equation}}
\newcommand{\bea}{\begin{eqnarray}}
	\newcommand{\eea}{\end{eqnarray}}
\newcommand{\beas}{\begin{eqnarray*}}
	\newcommand{\eeas}{\end{eqnarray*}}

\newcommand{\bct}{\begin{center}}
	\newcommand{\ect}{\end{center}}

\newcommand{\ba}{ {\bf a} }

\newcommand{\bd}{ {\bf d} }

\newcommand{\bJ}{ {\bf J} }

\newcommand{\bp}{ {\bf p} }

\newcommand{\bt}{ {\bf t} }

\newcommand{\bv}{ {\bf v} }

\newcommand{\bw}{ {\bf w} }
\newcommand{\bW}{ {\bf W} }
\newcommand{\bx}{ {\bf x} }
\newcommand{\bX}{ {\bf X} }
\newcommand{\by}{ {\bf y} }
\newcommand{\bY}{ {\bf Y} }

\def\bbR{\mathbb{R}}
\def\bbN{\mathbb{N}}

\newcommand{\bbeta}{ {\boldsymbol \beta} }

\newcommand{\bgamma}{ {\boldsymbol \gamma} }

\newcommand{\bzero}{ {\boldsymbol 0} }

\makeatletter
\newcommand*\bigcdot{\mathpalette\bigcdot@{.5}}
\newcommand*\bigcdot@[2]{\mathbin{\vcenter{\hbox{\scalebox{#2}{$\m@th#1\bullet$}}}}}
\makeatother

\def\bzero{\mathbf{0}}
\def\bone{\mathbf{1}}
\def\argmin{\operatornamewithlimits{argmin}}

\def\bbeta{\boldsymbol{\beta}}

\def\cov{\mathrm{Cov}}

\def\bd{\mathbf{d}}

\def\E{\mathrm{E}}
\def\bgamma{\boldsymbol{\gamma}}

\def\bJ{\mathbf{J}}

\def\bx{\mathbf{x}}

\def\bv{\mathbf{v}}

\def\var{\mathrm{Var\,}}

\def\bX{\mathbf{X}}
\def\bY{\mathbf{Y}}

\def\cF{\mathcal{F}}
\def\cI{\mathcal{I}}

\def\cG{\mathcal{G}}

\def\rT{\mathrm{T}}

\newcommand{\cX}{{\cal X}}

\newcommand{\cT}{{\cal T}}

\startlocaldefs
\theoremstyle{plain}

\newtheorem{theorem}{Theorem}[section]
\newtheorem{lemma}[theorem]{Lemma}
\newtheorem{corollary}{Corollary}
\theoremstyle{remark}

\endlocaldefs

\begin{document}

\begin{frontmatter}
 \title{U-learning for Prediction Inference via Combinatory Multi-Subsampling:  With Applications to LASSO and Neural Networks}
\runtitle{U-Learning for Prediction Inference}

\begin{aug}
\author[A]{\fnms{Zhe}~\snm{Fei}\ead[label=e1]{zhef@ucr.edu}}
\and
\author[B]{\fnms{Yi}~\snm{Li}\ead[label=e2]{yili@umich.edu}}
\address[A]{Department of Statistics, UC Riverside\printead[presep={,\ }]{e1}}

\address[B]{Department of Biostatistics, University of Michigan\printead[presep={,\ }]{e2}}
\end{aug}

\begin{abstract}

  We introduce a novel U-learning approach  via combinatory multi-subsampling for making ensemble predictions and   constructing confidence intervals for predictions of continuous outcomes when existing asymptotic methods are not applicable. More specifically, our approach conceptualizes the ensemble estimators within the framework of generalized U-statistics and invokes the H\'ajek projection for deriving the variances of predictions and constructing confidence intervals with valid conditional coverage probabilities. We apply our approach to two commonly used predictive algorithms, LASSO for high dimensional linear models and deep neural networks (DNNs) for non-linear models. We illustrate the validity of inferences with extensive numerical studies. 
  We have applied these methods to predict the DNA methylation age  of patients with various health conditions using over 37,000 CpG (Cytosine-phosphate-Guanine) sites within their genome, aiming to accurately characterize the aging process and potentially guide anti-aging interventions. 
\end{abstract}



\end{frontmatter}


	\section{Introduction}\label{sec1}

 Epigenetic aging clocks, composed of CpG sites and their DNA methylation levels, are designed to predict the DNA methylation age of subjects. While strongly correlated with  chronological age,   DNA  methylation age   provides a more nuanced reflection of the aging process, indicating whether a person is aging slower or faster than  expectated.  There is a pressing need to obtain accurate predictions as well as to quantify the uncertainty of these predictions. More broadly, the challenge lies in making predictions and inferences with complex, often less-interpretable machine learning models, a common issue in modern learning tasks.
 
  Much progress has been achieved in high-dimensional inference, particularly in the realm of estimating parameters whose dimensionality surpasses the sample size. The methods encompass de-biased approaches \citep{zhang2014confidence, javanmard2014confidence}, post-selection exact inference \citep{lee2016exact, belloni2016post}, and many more \citep{ning2017general, fei2019drawing, zhu2018linear, fei2021estimation}. Moreover, extensions to non-linear models have also been explored \citep{van2014asymptotically, kong2021high, fei2021inference}. { However, none of these works address the uncertainty associated with individual predictions, as they aim to make inferences about model parameters or the effects of predictors, which fundamentally differs from inferring subject-specific predictions.}
  
   A notable area that connects high-dimensional inference with prediction inference is the task of inferring treatment effects amidst high-dimensional confounders \citep{belloni2014high, belloni2014inference,avagyan2021high}. This is  because treatment effects are often estimated as the differences in outcomes under different treatment assignments.  Important developments include the methods based on regression trees and random forests  \citep{wager2018estimation},  debiasing \citep{athey2018approximate}, and  estimating dynamic treatment effects within marginal structural models  \citep{bradic2021high}.   
   
Recent research in prediction inference yields several techniques, including a residual-based bootstrap method with adjustment to guarantee conditional validity and marginal coverage \citep{zhang2023bootstrap}; the distribution-free conformal inference \citep{lei2018distribution, angelopoulos2021gentle,kato2023review}; extension of the Jackknife inference \citep{kim2020predictive,barber2021predictive}; applications of sub-sampling and U-statistics to regressions and neural networks \citep{schupbach2020quantifying,wang2022quantifying};
  and lower upper bound estimation method for  neural networks-based prediction intervals \citep{khosravi2010lower}. 
 Extending beyond continuous outcomes, \cite{guo2021inference} focuses on inference for the case probability in high-dimensional logistic regressions, using an extension of the debiased/projection estimator.
 Limitations are present with these methods. For example,  the bootstrap-based approach detailed in \cite{zhang2023bootstrap} is primarily applicable to linear regression models, which may not cover the full spectrum of prediction scenarios; \cite{schupbach2020quantifying}, \cite{wang2022quantifying} and \cite{khosravi2010lower}  provide limited theoretical justifications of their methods.  \cite{mentch2016quantifying} derives theoretical results and properties of U-statistics for quantifying prediction uncertainty in random forests. However, it is unclear whether these theories can be directly applied to Lasso or deep learning algorithms.  Various conformal prediction methods often impose   assumptions like exchangeability, strongly mixing errors, or covariate shift \citep{papadopoulos2007conformal,tibshirani2019conformal,romano2020classification,chernozhukov2021distributional}, and tend to be conservative with  wide intervals. Moreover, since conformal prediction  derives the conformity score based on the \emph{observed outcomes}, it fails to distinguish measurement errors within the observations. This would result in less efficiency compared to  methods that directly model the distribution of \emph{true outcomes}.  

To tackle these challenges, we aim to predict future outcomes and accurately assess the associated uncertainty. This is particularly pertinent with less structured algorithms where established inference techniques are not readily available. In the motivating epigenetic project, we use penalized regressions and deep neural networks to predict subject-specific methylation ages, and it is crucial to determine whether these predictions are statistically significantly different from the chronological ages.
To improve prediction accuracy with high dimensional features, we propose an ensemble approach that originates from U-statistics, and resort to the H\'ajek projection properties for variance estimation and inference. The ensemble prediction improves the efficiency and leads to tractable asymptotics and valid inferences. Intuitively, the proposed approach extends random forests, which ensemble decision trees, to the ensembles of regressions and neural networks.

Our method offers several  advantages over existing approaches: 
 i) our ensemble learners are asymptotically normal, a result proven through a careful derivation from generalized U-statistic properties \citep{mentch2016quantifying}; ii) our variance estimator is model-free, requiring minimal assumptions, owing to the Hájek projection properties \citep{hoeffding1992class,hajek1968asymptotic,wager2018estimation}; iii) our confidence intervals are tailored to individual future subjects, providing \emph{conditional} coverage, in contrast to \emph{marginal} coverage at the population level, thereby enhancing  applicability. A contribution  is the formulation of ensemble predictions,  based on  Lasso or neural networks, as incomplete U-statistics,  leading to the tractable asymptotics \citep{janson1984asymptotic,frees1989infinite,lee2019u}. 

In Section 2, we introduce the models and proposed methods. We discuss the theoretical properties in Section 3 and  present the numerical examples in Section 4. Section 5 analyzes the DNA methylation data, illustrating the development of new epigenetic aging clocks using our proposed methods. Section 6  concludes the paper with  remarks and outlines of future  directions. 

\section{U-learning for Prediction Inference}\label{sec2}
	
	Assume  the observed data of size $n$, $\cT_n = \{Z_i = (y_i, \bx_i): i=1, \ldots, n\}$, are independently and identically distributed (i.i.d.) copies of $(y, \bx)$, with  $y \in \bbR^{1}$ and $\bx \in \bbR^{1 \times p}$ 
 satisfying
	\begin{equation}\label{m1}
		y = f_0(\bx) + \varepsilon,
	\end{equation}
	where $f_0: \bbR^{p} \rightarrow \bbR$ is a  function quantifying the dependence of the expectation of $y$ on $\bx$,  and $\varepsilon$ is a mean zero error term  that is  independent of $\bx$. We assume that $\bx$ follows an unspecified  distribution of $\cX$.
{ We define a compatible prediction problem: given a fixed 
 $\bx_*$ in the support of $\cX$, 
 suppose the associated outcome $y_*$ follows (\ref{m1}) but is unobserved. The goal of the paper centers on  predicting $\E(y_*| \bx_*)= f_0( \bx_*)$ based on $\cT_n$.
If we obtain a `good' estimate of $f_0(\cdot)$, denoted by  $\wh{f}(\cdot)$, from Lasso or a deep neural network (DNN), we would predict $f_0(\bx_*)$ via $\wh{f}(\bx_*)$. To quantify the prediction uncertainty, we will derive a confidence interval $\left( \wh{L}(\bx_*), \wh{U}(\bx_*) \right)$ that covers  $f_0(\bx_*)$ with a target probability of $ 0 < 1-\alpha < 1$.} That is,   
\[
\Pr\left[ f_0(\bx_*) \in \left( \wh{L}(\bx_*), \wh{U}(\bx_*) \right) \right] \rightarrow 1 - \alpha, \quad\text{as } n\rightarrow \infty.
\]
Intuitively, if  $\wh{y}_* = \wh{f}(\bx_*)$ is a consistent estimate of $f_0(\bx_*)$, 
and is asymptotically normal, an approximate confidence interval based on a normal distribution can be obtained by estimating the {\em prediction error variance}, $ \sigma_*^2 \equiv \var \left( \wh{f}(\bx_*)\right)$.  In machine learning, however, making predictions based on a single dataset typically leads to significant variations, and  its large sample properties are often intractable \citep{andreassen2020asymptotics,zavatone2021asymptotics}. 

As a general approach, we propose below a combinatory multi-subsampling (CMS) scheme that provides an ensemble prediction with tractable asymptotics and model-free variance estimates.
This scheme can be applied to many commonly used machine learning algorithms, including Lasso and neural networks. 
For the index set \(\cI = \{1,\ldots, n \}\) of the training data \(\cT_n\),  there are $B^*=\binom{n}{r}$ combinations of unique subsets of size $r (< n)$. We enumerate these unique subsets from 1 to
 $B^*$ (e.g., in the lexical order), and denote each one as {$\cI^{\mathring{b}} = \{i_1,\ldots, i_r \}$ with $i_1 <\ldots <i_r$ and $\mathring{b}=1, \ldots, B^*$. 
\begin{enumerate}
   \item  For each $\mathring{b} = 1,2,..,B^*$, subsample $r$ observations (without replacement) with the index set {$\cI^{\mathring{b}} = \{i_1,\ldots, i_r \}$}   from the original samples $\cT_n$. Denote by  
   $\cI_V^{\mathring{b}} =  \cI\setminus \cI^{\mathring{b}}$.
   
   \item  Train the model with observations indexed by $\cI^{\mathring{b}}$, and use 
   observations indexed by $\cI_V^{\mathring{b}}$ for model selection/hyperparameter tuning. Denote the trained model by $\wt{f}^{\mathring{b}}$.  

    \item Apply $\wt{f}^{\mathring{b}}$ to a testing point $\bx_*$ and denote the prediction as $\wt{f}^{\mathring{b}}(\bx_*)$. Compute the ensemble prediction as $$\wh{f}^{B^*}(\bx_*) = \frac{1}{B^*}\sum_{{\mathring{b}}=1}^{B^*}\wt{f}^{\mathring{b}}(\bx_*).$$ 
    \item Compute the variance estimator by infinitesimal jackknife (introduced in the next subsection) as $\wh{\sigma}^2_*$ and derive the confidence interval for prediction (CIP) as
    \[
    \left( \wh{f}^{B^*}(\bx_*) - c_\alpha \wh{\sigma}_*, 
    \wh{f}^{B^*}(\bx_*) + c_\alpha \wh{\sigma}_*
    \right),
    \]
where $c_\alpha$ controls the $(1-\alpha)$ confidence level with  a given $0<\alpha<1$.
\end{enumerate}
We will use generalized U-statistic theories to derive the asymptotic properties of $\wh{f}^{B^*}(\bx_*)$ and to draw inference,  
thus our procedure is termed \emph{U-learning} for prediction and inference. In practice, we allow $r$  to increase along with $n$ (denoted as $r_n$). As computing all $B^*$ models is not feasible even with moderate $r$ and $n$, we propose a stochastic approximation with a suitable $B$ for implementing the U-learning procedure.

 \subsection{Application to prediction with the LASSO}
 Consider a linear version of Model (\ref{m1}), i.e.,
 \begin{equation}\label{m_linear}
 \begin{aligned}
		y_i = f_0(\bx_i) + \varepsilon_i
  = \beta_0^0 + \bx_i\bbeta^0 + \varepsilon_i,  \,\, i=1, \ldots, n, 
 \end{aligned}
 \end{equation}
where $\bx_i = (x_{i1}, \ldots, x_{ip})$, and $f_0(\bx)= \beta_0^0 + \bx \bbeta^0$. 
{Denote the response vector and design matrix of $\cT_n$ by $\bY= (y_1, \ldots, y_n)^T$ and   $\bX= (\bx_1^\rT, \ldots, \bx_n^\rT)^\rT$, respectively.}
For $\mathring{b} = 1,2,..,B^*$,
denote by $\cT^{\mathring{b}} = \left(\bY_{\cI^{\mathring{b}}}, \bX_{\cI^{\mathring{b}}} \right)$, where $\bY_{\cI^{\mathring{b}}}$ and $\bX_{\cI^{\mathring{b}}}$ are, respectively, a subvector of $\bY$ and
 a submatrix of $\bX$ with rows indexed by   $\cI^{\mathring{b}}$. 
 In  high dimensional settings, where classic  linear regressions are not applicable, we apply Lasso  to fit model (\ref{m_linear}) and obtain the estimates of $\beta_0^0$ and  $\bbeta^0$:
	\begin{equation}\label{lasso1}
	   \left(\wt{\beta}_0^{\mathring{b}}, \wt{\bbeta}^{\mathring{b}} \right) = \argmin_{ {(\beta}_0, \bbeta): ||\bbeta||_1 < K}\| \bY_{\cI^{\mathring{b}}} -  {\beta}_0 \bone_r- \bX_{\cI^{\mathring{b}}}\bbeta \|^2_2,
	\end{equation}
	where $\bone_r$ is an $r$-dimensional vector of 1's,
and  $K$ is a pre-defined positive constant. This formulation is  equivalent to the unconstrained optimization problem of ``Loss + Penality''. Then for any fixed $\bx_*$ in the support of $\cX$, the Lasso prediction based on the subsample $\cT^{\mathring{b}}$ is 
\begin{equation}\label{lp_sub}
    \wt{y}^{\mathring{b}}_* = \wt{f}\left(\bx_*; \cT^{\mathring{b}}  \right) = \wt{\beta}_0^{\mathring{b}} + \bx_*\wt{\bbeta}^{\mathring{b}}.
\end{equation}
In general, the Lasso prediction $\wt{y}_*^{\mathring{b}}$ based on a single dataset incurs considerable variation, and its distribution is not tractable \citep{fu2000asymptotics}, even when the sample size is large, thus making inference rather challenging. A key observation, however, is  that with a fixed $K$, $\wt{y}_*^{\mathring{b}}$ is permutation symmetric with respect to $\cI^{\mathring{b}}$.
Hence, the ensemble prediction \begin{equation}\label{y_inf}
	    \wt{y}_*^\infty = \left( \sum_{{\mathring{b}}=1}^{B^*}\wt{y}_*^{\mathring{b}}\right) /\binom{n}{r}
     \end{equation}
	is indeed a generalized U-statistic \citep{lee2019u}, 
 which can be shown to be a consistent estimator of $f_0(\bx_*)$ and is asymptotically normal. Moreover, its U-statistic structure
  leads to an estimate of its variance via the H\'ajek projection, facilitating construction of confidence intervals. The formulation of (\ref{y_inf}) is general; 
  when $r = n-1$ and thus $B^*= n$,  (\ref{y_inf}) includes the Jackknife estimate as a special case.
 

Computing $\wt{y}_*^\infty$ requires fitting $\binom{n}{r}$ regression models, which is impractical even with moderate $r$ and $n$.  We propose a stochastic approximation of (\ref{y_inf}) with a sufficiently large $B$.  That is,  randomly draw $B$  size-$r$ subsamples from the original data  $\cT_n$. In term of probability, 
this can be described as independently generate $b_1, \ldots, b_B$, where each $b_j$ is chosen from  1 to ${\binom{n}{r}}$ with equal probability,  i.e., 
$\Pr(b_j= \mathring{b} ) = 1/{\binom{n}{r}}$ for $\mathring{b} \in  \{1, \ldots, {\binom{n}{r}}\}$,  and the $j$-th  size-$r$ subsample is indexed by $\cI^{b_j},j = 1, \ldots, B$.

With this notion and as described in Algorithm \ref{alg_pred}, we propose the U-learning prediction
	\begin{equation}\label{fhat_B}
	\wh{y}^B_* 
 = \frac{1}{B}\sum_{j=1}^{B}
( \wt{\beta}_0^{b_j} + \bx_*\wt{\bbeta}^{b_j})
 = \frac{1}{B}\sum_{j=1}^{B} \wt{y}^{b_j}_*,
	\end{equation}
	where each $\wt{\beta}_0^{b_j}$ and $\wt{\bbeta}^{b_j}$ are the Lasso estimates from (\ref{lasso1}) based on the sub-sampled data $\cT^{b_j}$. 
By the law of large numbers, (\ref{fhat_B}) approximates (\ref{y_inf}) well as $B$ increases.  With a U-statistic framework, we term the prediction (\ref{fhat_B}) \emph{Lasso U-learning}.
 
  

As (\ref{fhat_B}) is an ensemble prediction based on subsampling, we apply the infinitesimal jackknife method \citep{wager2018estimation,fei2021estimation} to obtain a variance estimator of 
 $\wh{f}^B(\bx_*)$ as
	\begin{equation}\label{formula:var}
	   \wh{\sigma}_*^2 = 
    \frac{n-1}{n} \left(\frac{n}{n-r} \right)^2 \sum_{i=1}^{n}\wh{\cov}_{i,*}^2 
	\end{equation}
	where the sampling covariance is
	\begin{equation*}\label{cov}
	    \wh{\cov}_{i,*} = \frac{\sum_{j=1}^B (J_{b_ji} - J_{\cdot i})(\wt{y}^{b_j}_* - \wh{y}^B_* ) }{B},
	\end{equation*}
	where $J_{b_ji} = I(i\in \cI^{b_j})$ and $J_{\cdot i}= \frac{1}{B} \sum_{j=1}^B J_{b_ji}$.

\begin{algorithm}
	\caption{Lasso U-learning prediction and inference \label{alg_pred}}
	\begin{algorithmic}[1]
		\REQUIRE Subsampling size  $r$, and number of subsamples $B$
		\INPUT Training data $\cT_n$, Lasso constraint $K$, and a fixed test point $\bx_*$
		\OUTPUT Prediction  $\wh{y}^B_*$ and $100(1-\alpha)\%$ confidence interval $\left( \wh{L}(\bx_*), \wh{U}(\bx_*) \right)$
	\FOR{$j=1,2,\ldots,B$}
		\STATE 
  Random subsample data of size $r$ without replacement.  Denote by $\cT^{b_j}$ the resampled data (indexed by $\cI^{b_j}$), and save the sampling vector $\bJ_{b_j}
  = (J_{b_j1}, \ldots, J_{b_jn})$
		\STATE Fit Lasso on $\cT^{b_j}$ with the constraint constant $K$ and output  $ \wt{\beta}_0^{b_j}, \wt{\bbeta}^{b_j}$ 
		\STATE Predict $\wt{y}^{b_j}_* =  \wt{\beta}_0^{b_j}+ \bx_*\wt{\bbeta}^{b_j}$
		 \ENDFOR
		\STATE Compute $\wh{y}^B_*$ by (\ref{fhat_B})  and $\wh{\sigma}^2_*$ by  (\ref{formula:var}) 
  \STATE Compute a $100(1-\alpha)\%$ confidence interval for $f_0(\bx_*)$ as 
  \(
  \left(\wh{y}^B_* - z_{1-\alpha/2}\wh{\sigma}_*, \wh{y}^B_* + z_{1-\alpha/2} \wh{\sigma}_*\right)
  \)
	\end{algorithmic}
\end{algorithm}
    
 \subsection{Application to prediction with deep neural networks}

Deep neural networks have emerged as a powerful means for nonparametric regression.
We use   neural networks  to fit a nonparametric version of Model (\ref{m1}), where  $f_0(\cdot):[0,1]^p \rightarrow \bbR$ is an unknown truth satisfying certain smoothness conditions.
The basic unit of a neural network is a neuron, which takes input values, processes them using weights and biases (intercepts), and produces an output value. Denote the input values as a row vector $\bx$, the weights as column vector $\bw$ and the bias term as $a$. The output of a neuron is calculated by applying an activation function, denoted by  $\sigma: \bbR \rightarrow \bbR$, to a weighted sum of its inputs plus an intercept (bias): $v = \sigma( \bx\bw + a)$. For the purpose of the theoretical derivations, we restrict to the ReLU activation function in this paper, i.e., $\sigma(x) = \max(x, 0)$. 

A multilayer neural network or a deep neural network (DNN) consists of multiple layers of interconnected neurons with a specific network architecture $(L, \bp)$, where $L$ is the number of hidden layers, and $\bp = \left(p_0, p_1, ..,p_{L+1} \right) \in \bbN^{L+2}$ the width vector \citep{schmidt2020nonparametric}. Then a DNN can be expressed as
\begin{gather}
    f: \bbR^{p_0} \rightarrow \bbR^{p_{L+1}}, \nonumber \\
    \bx \mapsto f(\bx) = \bW_L\sigma_L(\bW_{L-1}\sigma_{L-1}(..\bW_1\sigma_1(\bW_0\bx^\rT + \ba_0)+\ba_1)..+\ba_{L-1})+\ba_L, \label{nnclass}
\end{gather}
where $\bW_\ell$'s are $p_{\ell+1} \times p_\ell$ weight matrices and $\ba_\ell$'s are the intercept vector, and $\sigma_\ell(\cdot)$'s are activation functions operating on vectors
element-wise. In our setting,  $p_0 = p$ and $p_{L+1} = 1$.
We denote $\cF(L,\bp)$ as the class of functions defined in (\ref{nnclass}) 
 with $L$ hidden layers and the width vector $\bp$.
There is a natural extension of Algorithm \ref{alg_pred} with DNNs, termed DNN U-learning, which  is detailed in Algorithm \ref{alg_nn}. 

\begin{algorithm}
	\caption{DNN U-learning prediction and inference\label{alg_nn}}
	\begin{algorithmic}[1]
		\REQUIRE DNN architecture $(L,\bp)$, subsample size  $r < n$, number of subsamples $B$
		\item[\textbf{Input:}] Training data $\cT_n$, a test point $\bx_*$
		\item[\textbf{Output:}] Prediction $\wh{y}^B_*$ and $100(1-\alpha)\%$ confidence interval $\left( \wh{L}(\bx_*), \wh{U}(\bx_*) \right)$
	\FOR{$j=1,2,\ldots,B$}
		\STATE  Random subsample data of size $r$ without replacement.  Denote by $\cT^{b_j}$ the resampled data (indexed by $\cI^{b_j}$), and save the sampling vector $\bJ_{b_j}
  = (J_{b_j1}, \ldots, J_{b_jn})$
		\STATE Fit DNN$(L,\bp)$ on $\cT^{b_j}$ with $\cT\setminus \cT^{b_j}$ as the validation set  
  \STATE Denote the fitted model as $\wt{f}^{b_j}$ 
		 \ENDFOR
		\STATE Compute $\wh{y}^B_* = \frac{1}{B}\sum_{j=1}^{B} \wt{f}^{b_j}(\bx_*)$ and $\wh{\sigma}^2_*$ by (\ref{formula:var})
   \STATE Compute the confidence interval for $f_0(\bx_*)$ as 
  \(
  \left(\wh{y}^B_* - z_{1-\alpha/2}\wh{\sigma}_*, \wh{y}^B_* + z_{1-\alpha/2} \wh{\sigma}_*\right)
  \)
	\end{algorithmic}
\end{algorithm}

\section{Theoretical Properties}\label{sec3}

 \subsection{Lasso U-learning} 
	We  will show that the U-learning prediction (\ref{fhat_B}) 
 is consistent and asymptotically normal, as well as the consistency of the variance estimator (\ref{formula:var}).
We start by stating the \emph{Restricted Eigenvalue (RE)} condition which is commonly used  in the  Lasso literature and will be used for formulating our regularity conditions. Given a set \( S \subseteq \{1, 2, ..., p\} \) and a positive number \( c_\alpha \), $\bX$ is said to satisfy the RE condition if there exists a constant \( \kappa > 0 \) such that for all vectors \( \Delta \in \mathbb{R}^p \) satisfying \( \| \Delta_{S^c} \|_1 \leq c_\alpha \| \Delta_S \|_1 \), it holds that 
$\frac{1}{n} \| \bX\Delta \|_2^2 \geq \kappa \| \Delta \|_2^2.$
The RE condition essentially controls  multicollinearity among the predictors, and the eigenvalues of submatrices of $\bX^\rT\bX$. 
As such,  we state the following conditions  required for our theorems.
\begin{enumerate}[label=(C\arabic*)]
    \item \label{C1} The errors $\varepsilon_i$'s are independent of $\bx_i$'s and i.i.d. with  mean zero and a finite variance $\sigma^2_\varepsilon$.
    Further, there exists a constant $M>0$, such that $\|\bx_{i}\|_\infty < M$ almost surely for all $i$. The design matrix $\bX$ satisfies the RE condition with constants $(\kappa, 3)$.
    \item \label{C2} The subsampling size $r_n$ satisfies 
     $  \lim_{n\rightarrow \infty} r_n = \infty$ and
  $\lim_{n\rightarrow \infty} r_n \log(p) /n = 0$. 
 (In the ensuing theoretical development, we introduce a subscript $n$ to  $r$ to emphasize its reliance on the sample size $n$.)
    \item \label{C3} There exists a constant $K_0$, such that $|\bbeta^0|_1 =  \sum_{j=1}^{p} |\beta_j^0| < K_0$. Let $\{ K_n\}_{n\ge 1}$ be a sequence of constants satisfying $K_n > K_0$  for the Lasso problem (\ref{lasso1}), where given the training set, the same $K_n$ is used to fit (\ref{lasso1}) on all subsamples $\cT^b, b=1,2,..,B$.
\end{enumerate}
Condition \ref{C1} is standard in restricting the input covariates of the true model. {Condition \ref{C2} stipulates the  relationship between the the order of  $r_n, n$ and $p$. 
} Condition \ref{C3} is needed for the prediction consistency of Lasso and ensures that there is no additional randomness in fitting the Lasso on $\cT^b$. 
 Moreover, {diverging from the conventional high-dimensional literature, we do not impose sparsity on the coefficient vector ${\bbeta}^0$, as our focus is on predictive inference rather than drawing inferences about individual effects.} We  first present a result on  prediction consistency of Lasso
\citep{chatterjee2013assumptionless}.
	\begin{lemma}\label{lem1} Under  \ref{C1}, and for $K_n > K_0$ as defined in \ref{C3}, the prediction $\wt{y}_*^{\mathring{b}}$ based on  (\ref{lasso1}) satisfies
	\begin{equation*}\label{lasso_consist}
	    \E\left[(\wt{y}_*^{\mathring{b}} - f_0(\bx_*))^2\right] \le 2K_n M\sigma_\varepsilon \sqrt{\frac{2\log(2p)}{n}} + 8K_n^2M^2\sqrt{\frac{2\log(2p^2)}{n}}.
	\end{equation*}
	\end{lemma}
{\em Remark:} To ensure the consistency of the ensemble prediction, the Lemma establishes the prediction consistency for each subsample under conditions much weaker than for the ``estimation consistency.'' For example, Lasso yields shrinkage coefficient estimators that are biased and require stringent conditions for consistent estimates or model selection \citep{zhao2006model,chatterjee2011strong};
in contrast, Lasso prediction consistency can be readily obtained from past works with much weaker assumptions \citep{chatterjee2013assumptionless,bartlett2012l,rigollet2011exponential,buhlmann2011statistics}. 

Extending the discussion of the variance of a  U-statistic in
\cite{hoeffding1992class}, 
we define
\[
\xi_{1,r_n}(\bx_*) = \cov\left(\wt{f}(\bx_*; Z_1,Z_2,..,Z_{r_n}), \wt{f}(\bx_*; Z_1,Z_2'..,Z_{r_n}') \right)
\]
as the covariance of the Lasso predictions
(\ref{lp_sub})
on two subsamples with one sample in common. Because the data points $Z_i, Z_i'$'s follow the same data generation process, and given that both predictions are asymptotically consistent, it would be reasonable to assume  $\liminf \xi_{1,r_n}(\bx_*) > 0$.
As such, we  present the first main result on the asymptotic normality of the Lasso U-learning prediction. 

 \begin{theorem}\label{thm_normality}
	Given the training data $\cT_n = \{Z_i = (y_i, \bx_i): 1\le i \le n\}$ drawn from model (\ref{m_linear}) and a fixed  $\bx_*$ in the support of $\cX$, the distribution generating $\bx_i$, let $\wh{y}_*^B = \wh{f}^B(\bx_*)$ be  the U-learning prediction defined as in (\ref{fhat_B}) and Algorithm \ref{alg_pred}.  
 Then under   \ref{C1} - \ref{C3}, as $n \rightarrow \infty$ and $n/B \rightarrow 0$,
 \begin{itemize}
     \item [i)]	 \begin{equation*}
	    \frac{\sqrt{n}\left(\wh{y}_*^B - f_0(\bx_*) \right) }{ v_n(\bx_*) } \xrightarrow{d} N(0,1)
	\end{equation*}
	where  $v_n(\bx_*) = \sqrt{r_n^2 \xi_{1,r_n}(\bx_*) }$.
\item[ii)]     $n\wh{\sigma}_*^2/ v_n^2 \xrightarrow{p} 1$.
 \end{itemize}
	\end{theorem}

Theorem \ref{thm_normality} establishes the asymptotic normality of the U-learning prediction $\wh{y}_*^B$ and the consistency of the infinitesimal jackknife variance estimator $\wh{\sigma}_*^2$. {These results form the foundation for conducting statistical inference.
The key advantage of employing the infinitesimal jackknife approach lies in that we do not need to compute or estimate 
$v_n(\bx_*)$ or $\xi_{1,r_n}(\bx_*)$ analytically; instead, we can compute  its asymptotic equivalence $\wh{\sigma}_*^2$, which is embedded in the resampling scheme.
Also, when proving this theorem, 
we show and use the asymptotic normality of the H\'ajek projection of $\wh{y}_*^B$; 
see the \hyperref[appn]{Appendix}.}
The asymptotic normality ensures valid construction of confidence intervals for future predictions, as stated in the Corollary.
 \begin{corollary}\label{thm_var}
 Under the same setting as in Theorem \ref{thm_normality}, 
and with $\wh{\sigma}_*^2$  defined in (\ref{formula:var}), the following confidence interval is asymptotically valid with a target probability of $1-\alpha$  $(0 <\alpha <1)$. That is, as \(n\rightarrow \infty,\)
 \(
 \Pr\left[ f_0(\bx_*) \in \left( \wh{L}(\bx_*), \wh{U}(\bx_*) \right) \right] \rightarrow 1 - \alpha\), 
where $$
 \wh{L}(\bx_*) =\wh{y}^B_* - z_{1-\alpha/2}\wh{\sigma}_*, \, \quad
\wh{U}(\bx_*)=
     \wh{y}^B_* + z_{1-\alpha/2}\wh{\sigma}_*,$$
 and  $z_{1-\alpha/2}$ is the $(1-\alpha/2)$-th quantile of the standard normal distribution.
\end{corollary}


	

\subsection{DNN U-learning}

Since the Universal Approximation Theorem \citep{hornik1989multilayer}, much work has been accomplished on the prediction accuracy and rate of convergence of neural networks; see \cite{mccaffrey1994convergence,bolcskei2019optimal,schmidt2020nonparametric}, among many others. The performance of a DNN  is typically measured by the bound on the prediction error, 
\begin{equation*}
    R(\wh{f}_n, f_0) = \E_{f_0} \left[ \left( \wh{f}_n(\bx_*) - f_0(\bx_*) \right)^2 \right],
\end{equation*}
where $\wh{f}_n$ is the DNN estimator based on the training set $\cT_n$, $\E_{f_0}$ indicates that the expectation is taken under  the true function $f_0$ under model (\ref{m1}),
and $\bx_*$ is a fixed testing point. While the rate of convergence would vary depending on different model assumptions and DNN architectures, various previous works have provided results on the prediction consistency. For example, \cite{mccaffrey1994convergence} on shallow neural networks (one hidden layer) has derived a convergence rate of $n^{-2\gamma/(2\gamma+p+5)}$, where $\gamma$ is the smoothness parameter and $p$ is the number of covariates.
Additional insights for multilayer neural networks can be found in \cite{kohler2005adaptive}, which achieves the nonparametric rate $n^{-2\gamma/(2\gamma+p)}$ for $\gamma$-smooth functions $f_0$ and $\gamma\le 1$ with two-layer NNs using the sigmoid activation function. Extensions of this work include \cite{bauer2019deep,kohler2021rate}. 

In order to derive the convergence rate of $\wh{f}_n\in \cF(L,\bp)$ [defined underneath (\ref{nnclass})],  we first introduce
$$
    \Delta_n(\wh{f}_n, f_0) = \E_{f_0} \left[ \frac{1}{n}\sum_{i=1}^n \left( y_i - \wh{f}_n(\bx_i)  \right)^2 - \inf_{f\in \cF(L,\bp) } \frac{1}{n}\sum_{i=1}^n \left( y_i - f(\bx_i)  \right)^2  \right].
$$
The sequence $\Delta_n(\wh{f}_n, f_0)$ measures the difference between the expected empirical risk of $\wh{f}_n$ and the global minimum over all networks in the class. Therefore, $\Delta_n(\wh{f}_n, f_0) = 0$ if $\wh{f}_n$ is an empirical risk minimizer.
For the theoretical derivations, we refer to the framework established by \cite{schmidt2020nonparametric}, which assumes $f_0 = g_q \circ g_{q-1}\circ \ldots \circ g_1 \circ g_0$ with $g_i: [a_i, b_i]^{d_i} \rightarrow [a_{i+1}, b_{i+1}]^{d_{i+1}}$.  Denote by \( g_i = (g_{ij})^T_{j=1,\ldots,d_{i+1}} \) the components of \( g_i \), and let \( t_i \) be the maximal number of variables on which each of the \( g_{ij} \) depends on. Thus, each \( g_{ij} \) is a \( t_i \)-variate function. We assume that each of the functions \( g_{ij} \) has H\"older smoothness \( \gamma_i \), denoted by \( g_{ij} \in C^{\gamma_i}_{t_i}([a_i, b_i]^{d_i}) \). Since \( g_{ij} \) is also \( t_i \)-variate, \( g_{ij} \in C^{\gamma_i}_{t_i}([a_i, b_i]^{t_i}) \), and the underlying function space is
\begin{gather*}
\cG(q, \bd, \bt, \bgamma, K) := \left\{ f = g_q \circ \cdots \circ g_0 : g_i = (g_{ij})_j : [a_i, b_i]^{d_i} \to [a_{i+1}, b_{i+1}]^{d_{i+1}}, \right. \\
\left. g_{ij} \in C^{\gamma_i}_{t_i}([a_i, b_i]^{t_i}) \text{ with } |a_i|, |b_i| \leq K \right\},
\end{gather*}
with \( \bd := (d_0, \ldots, d_{q+1}) \), \( \bt := (t_0, \ldots, t_q) \), \( \bgamma := (\gamma_0, \ldots, \gamma_q) \).
Further define the effective smoothness indices as 
\( \gamma^*_i := \gamma_i \prod_{\ell=i+1}^q (\gamma_\ell \wedge 1) \),  
and
$
\phi_n := \max_{i=0,\ldots,q} n^{-\frac{2\gamma^*_i}{2\gamma^*_i+t_i}}.
$
Lastly, 
to facilitate our ensuing theoretical development, we focus on a class of $s$-sparse and $F$-bounded networks, defined as
\begin{equation*}
    \cF(L, \bp, s, F) := \left\{f\in \cF(L, \bp): \sum_{j=0}^L \|\bW_j\|_0 \le s, \| f\|_\infty \le F  \right\}.    
\end{equation*}

We require the following conditions.
\begin{enumerate}[label=(D\arabic*)]
    \item \label{D1} The input features $\bx$ are bounded between $[0,1]$. 
    \item \label{D2} The subsampling size $r_n$ satisfies 
 $       \lim_{n\rightarrow \infty} r_n = \infty$ and $        \lim_{n\rightarrow \infty} r_n \log(p) /n = 0.$
    \item \label{D3} $\wh{f}_{n} \in \cF(L, \bp, s, F)$ with $L$, $\bp$ and $s$ satisfying 
    \begin{itemize}
        \item[(i)] $\sum_{i=0}^q \log_2(4t_i \vee 4\gamma_i) \log_2 {n} \le L \lesssim \log^\alpha {n}$ where $\alpha>1$,
        \item[(ii)] ${n} \phi_{n} \lesssim \min_{i=1,\ldots,L} p_i$,
        \item[(iii)] $s \asymp {n}\phi_{n} \log {n}$,
    \end{itemize}
        where $a_n \lesssim b_n$ means there exists a  constant $C>0$ such that $0<a_n \le Cb_n$ when $n$ is  sufficiently large, and $a_n \asymp b_n$ means both $a_n \lesssim b_n$ and  $b_n \lesssim a_n$ hold. 
    \item \label{D4} There exists a constant $C$ such that the expected empirical risk difference $\Delta_n(\wh{f}_n, f_0) \le C\phi_n L\log^2 n$.
\end{enumerate}
Conditions \ref{D1} and \ref{D2} are analogous to \ref{C1} and \ref{C2}. Condition \ref{D3} specifies the requirements on the network parameters relative to the sample size and smoothness parameters $\phi_n, t_i, \gamma_i$, including the number of layers $L$,  the boundness parameter $F$, and the sparsity $s$. Note that $s$ is allowed to increase with $n$, and in practice we can apply regularization and dropout layers in applications to control the network sparsity. The boundness of network is a technical condition, ensuring the convergence  of  DNN U-learning predictions.  
Condition \ref{D4} assumes that we could obtain a fit $\wh{f}_n$ that is ``close enough'' to the global minimum within the class of neural networks. 
Combining \ref{D3} and \ref{D4} leads to the consistency requirement on the neural network, as we re-write Theorem 1 of \cite{schmidt2020nonparametric} as a lemma below, giving the convergence rate of $\wh{f}_n$  essential for deriving the prediction inference properties.
\begin{lemma}\label{lem_nn}
Consider Model (\ref{m1}) and $f_0$ in the class $\cG(q, \bd, \bt, \bgamma, K)$. Let $\wh{f}_n$ be an estimator in the neural network class $\cF(L, \bp, s, F)$. 
If conditions \ref{D3} and \ref{D4} are satisfied, there exist constants $0<c^* < 1/2$ and $C'$ such that $\phi_n = o(n^{-c^*})$, 
and
    $$
    R(\wh{f}_n, f_0 ) \le  C'\phi_n L\log^2 n.
    $$
\end{lemma}
Lemma \ref{lem_nn} establishes the bound of  prediction errors  within the class $\cF(L,\bp, s,F)$, 
where $c^*$  is determined by the smoothness of  $f_0$. 
 In contrast to Theorem \ref{thm_normality} for Lasso U-learning, deriving the asymptotics for DNN U-learning predictors introduces a key distinction: {while Lasso solutions are uniquely determined by data and the regularization parameter $K_n$, fitting NNs 
 introduces additional randomness through gradient-based algorithms.} Consequently, adjustments to the proof are necessary, akin to the modifications outlined in Theorem 2 of \cite{mentch2016quantifying}. {Let  $\omega$ denote the randomness involved in fitting a DNN, such as shuffling in batches, or the randomness in stochastic gradient descent (SGD) or other optimization algorithms, which leads to a nuanced formulation:}
 \begin{gather*}
     \wt{f}^{b_j}(\bx_*) = \wt{f}^{(\omega_j)}\left( \bx_*|L, \bp, \cT^{b_j} \right), \\
     \wh{y}^B_* = \frac{1}{B}\sum_{j=1}^B \wt{f}^{b_j}(\bx_*) = \frac{1}{B}\sum_{j=1}^B \wt{f}^{(\omega_j)}\left( \bx_*|L, \bp,\cT^{b_j} \right),
 \end{gather*}
where $\omega_1, \ldots, \omega_B \sim_{i.i.d.} F_\omega$, and are independent of the observations $\cT_n$.  As specified in Algorithm \ref{alg_nn}, we apply a fixed architecture $(L, \bp)$ to all subsamples. The resulting ensemble prediction is  asymptotically normal after
applying Lemma \ref{lem_nn}.  
\begin{theorem}\label{thm_NN}
Suppose that the training data $\cT_n = \{Z_i = (y_i, \bx_i): 1\le i \le n\}$ are generated from model (\ref{m1}) with an unknown $f_0$, and   $\bx_*$ is a  fixed testing point that lies in the support of  $\cX$. {Assume the neural network fit on subsamples in Algorithm \ref{alg_nn} belongs to  $\cF(L, \bp, s, F)$ and satisfies  \ref{D4} for the subsample size $r_n$.} In addition, \ref{D1} -- \ref{D3} hold, as well as $\liminf \xi_{1,r_n}(\bx_*) > 0$ and $\forall j$
 \begin{equation} \label{add-cond}
\lim_{n\rightarrow \infty} \E  \left( \wt{f}^{(\omega_j)}\left( \bx_*|\cT^{b_j} \right) - \E_\omega \wt{f}^{(\omega_j)}\left( \bx_*|\cT^{b_j} \right) \right)^2 < \infty,
\end{equation}
{where $\E_\omega$ is the expectation taken over $F_\omega$.}
Then the U-learning prediction by Algorithm \ref{alg_nn} satisfies, as $n \rightarrow \infty, n/B \rightarrow 0$,
	\begin{equation*}
	    \frac{\sqrt{n}\left(\wh{y}_*^B - f_0(\bx_*) \right) }{ v_{n} (\bx_*) } \xrightarrow{d} N(0,1)
	\end{equation*}
	for $v_n = \sqrt{r_n^2 \xi_{1,r_n}(\bx_*)}$. 
\end{theorem}
{\em Remark:} The additional Condition (\ref{add-cond}), 
an analog to that in  \cite{mentch2016quantifying},  prevents the randomness, when fitting a DNN on a subsample dataset, from causing  much difference among predictions from the same subsample  as $n \rightarrow \infty$,  
and eventually controls the difference between $\wh{y}_*^B$ and its non-random counterpart.

Because the analytic form of $\xi_{1,r_n}$ is unavailable, we  estimate the variances  via infinitesimal jackknife. Recall $\wh{\sigma}_*^2$
as defined in (\ref{formula:var}). The following corollary explicitly writes out the confidence interval for predictions with NNs.  The proof is omitted due to its similarity with the proof of Corollary \ref{thm_var}.
\begin{corollary}\label{thm_NN_inf}
 Under the same setting as in Theorem \ref{thm_NN}, the following confidence interval is asymptotically valid with a target probability of $1-\alpha$  $(0 <\alpha <1)$. That is, as \(n\rightarrow \infty,\)
 \(
 \Pr\left[ f_0(\bx_*) \in \left( \wh{L}(\bx_*), \wh{U}(\bx_*) \right) \right] \rightarrow 1 - \alpha\), 
where $$
 \wh{L}(\bx_*) =\wh{y}^B_* - z_{1-\alpha/2}\wh{\sigma}_*, \, \quad
\wh{U}(\bx_*)=
     \wh{y}^B_* + z_{1-\alpha/2}\wh{\sigma}_*,$$
 and  $z_{1-\alpha/2}$ is the $(1-\alpha/2)$-th quantile of the standard normal distribution.
\end{corollary}

\section{Numerical Experiments}\label{sec4}

We conduct numerical experiments to assess the finite-sample performance of the proposed U-learning methods and compare them with an Oracle estimator (assuming the true active set known) and the conformal prediction interval \citep{chernozhukov2021distributional}. We  consider three examples: i) high dimensional linear models with the Lasso; ii) non-linear models with neural networks; iii) WHO Life Expectancy Data, a public dataset for training and evaluating neural networks. In each example, we randomly split the data into the training and testing sets,
fit the model on the training set, and  assess the prediction performance on the testing set via the average bias and mean absolute error (MAE) between the truth and the predicted values, i.e., for $\ell$ testing samples,
\( MAE = \frac{1}{\ell}\sum_{\ell = 1}^{\ell} \left| f_0(\bx_\ell) - \wh{f}^{B}(\bx_\ell) \right|. \)
We further report the average  interval length (AIL), i.e.,
$\frac{1}{\ell} \sum_{\ell=1}^{\ell}(U_\ell - L_\ell),$
and the  coverage probability (CP), i.e.,
 $\frac{1}{\ell} \sum_{\ell=1}^{\ell} I( f_0(\bx_\ell) \in [\wh{L}_\ell, \wh{U}_\ell])$. 
 In the third example, we replace $f_0(\bx_\ell)$ by the corresponding observed value as an approximation without  knowing  the ``truth."

	
	
{\bf Example 1 (Lasso U-learning with high dimensional linear truth).} To mimic the real-world DNA methylation data, we let $p=3000$, with $s_0 = 25$ non-zero signals taking values between $-1$ and $1.5$. 	We ran two scenarios with training sizes $n=500$ and $n=1000$, respectively, and one fixed testing set with sample size $\ell=200$.  Two hundred training sets were simulated following the linear model (\ref{m_linear}). 
To implement the proposed algorithm, we chose $K$ by 5-fold cross validation on the entire training set, used $B=500$ and subsampling with $r = n^\gamma$ for $\gamma = 0.8, 0.9$ and $0.95$. The oracle estimates of prediction and CIP were derived assuming the true active set was known, and the model  was fit in each subsample in Algorithm \ref{alg_pred} instead of the Lasso.
For comparison, we implemented the approach of sampling with replacement (SWR, i.e., we replaced subsamples with bootstrap samples of size $r=n$ in Algorithm \ref{alg_pred}), the naive bootstrap approach (prediction using the entire training set, while using $B$ bootstrap resamples to derive the prediction SEs and confidence intervals based on normal approximation),  
and the conformal prediction \citep{lei2018distribution}. 

{Table \ref{tab_sim1} summarizes  the results, highlighting several key observations. First, as expected, the oracle method presents the narrowest confidence intervals. Second, concerning the proposed U-learning approach, the average bias, MAE, and Standard Error (SE) in prediction exhibit variability with respect to the subsample size $r$, achieving optimal results at $r=n^{0.9}$ and $r=n^{0.95}$. Third, the performance of U-learning improves notably as the sample size $n$ increases from 500 to 1000, with computed SE values approaching empirical standard deviations and coverage probabilities of confidence intervals aligning more closely with the $0.95$ nominal level. Particularly, at $n=1000$ and with $r=n^{0.9}$ or $r=n^{0.95}$, its AIL (0.436 or 0.441) is comparable to that of the Oracle (0.360). Additionally, Figure \ref{fig_sim1} depicts the average coverage probability on the test samples (panels a, c) and the average prediction SEs (panels b, d) in both scenarios. The empirical coverage probability is close to the nominal coverage probability. Our approach can also effectively capture individual variation in prediction, as evidenced by the varying SE values.
Fourth, it is noteworthy that the SEs of the two bootstrap methods (SWR and naive) fail to match the empirical standard deviations and nominal coverage level even as the sample size increases. This is because the SWR variant does not fit the generalized U-statistic framework and the Naive Bootstrap does not yield asymptotically normal predictions.}
Finally,  for a fair comparison, we slightly modified  conformal prediction by using the true $f_0(\bx_*)$ instead of $y_*$ when deriving the conformity score. Nevertheless, we find that,  while conformal prediction intervals (PIs) yield comparable coverage probabilities, it is less efficient with much larger AIL.  
In terms of computation, the U-learning computation time is proportional to the training sample size $n$ and $B$. Empirically on an 8-core CPU machine with parallel computing, the average run time per training set is 88 seconds for $n=500$ and 272 seconds for $n=1000$.

\begin{table}[ht]
\centering
\caption{Simulation example 1: comparisons of prediction and inference for Lasso U-learning with various resample sizes, oracle prediction, and conformal prediction. $^*$Note: We used true $f_0(\bx_*)$ instead of $y_*$ to derive conformal PIs for fair comparisons; SE is not available for  conformal predictions.
} \label{tab_sim1}
\begin{tabular}{>{\rowmac}r>{\rowmac}r>{\rowmac}r>{\rowmac}r>{\rowmac}r>{\rowmac}r>{\rowmac}r<{\clearrow}}
  \hline
 & Bias & MAE & EmpSD & SE & CP & AIL \\ 
  \hline
  \multicolumn{7}{c}{$n=500, p=3000$} \\
  Oracle & 0.003 & 0.009 & 0.088 & 0.093 & 0.956 & 0.366 \\ 
$r=n^{0.8} = 144$ & 0.033 & 0.659 & 0.166 & 0.184 & 0.936 & 0.723 \\ 
\setrow{\bfseries}$r=n^{0.9} = 268$ & 0.013 & 0.328 & 0.119 & 0.137 & 0.952 & 0.538 \\ 
\setrow{\bfseries}$r=n^{0.95} = 366$ & 0.010 & 0.266 & 0.112 & 0.137 & 0.958 & 0.536 \\ 
$r=n$ (SWR) & 0.008 & 0.216 & 0.152 & 0.141 & 0.941 & 0.552 \\ 
 Naive Bootstrap & -0.003 & 0.213 & 0.283 & 0.240 & 0.781 & 0.941 \\ 
Conformal PI$^*$ & 0.018 & 0.349 & 0.211 & - & 0.961 & 1.884 \\ 
\hline
   \multicolumn{7}{c}{$n=1000, p=3000$} \\ 
  Oracle & -0.001 & 0.005 & 0.083 & 0.092 & 0.968 & 0.360 \\ 
$r=n^{0.8} = 251$ & 0.016 & 0.539 & 0.143 & 0.143 & 0.969 & 0.562 \\ 
\setrow{\bfseries}$r=n^{0.9} = 501$ & 0.009 & 0.290 & 0.105 & 0.111 & 0.951 & 0.436 \\ 
 \setrow{\bfseries}$r=n^{0.95} = 707$ & 0.007 & 0.236 & 0.101 & 0.112 & 0.947 & 0.441 \\ 
 $r=n$ (SWR)  & 0.005 & 0.170 & 0.160 & 0.146 & 0.965 & 0.571 \\ 
 Naive Bootstrap & 0.008 & 0.172 & 0.283 & 0.235 & 0.78 & 0.923 \\
Conformal PI & 0.009 & 0.295 & 0.175 & - & 0.941 & 1.48 \\ 
   \hline
\end{tabular}
\end{table}

 \begin{figure}
    \centering
    \includegraphics[width=\textwidth]
    {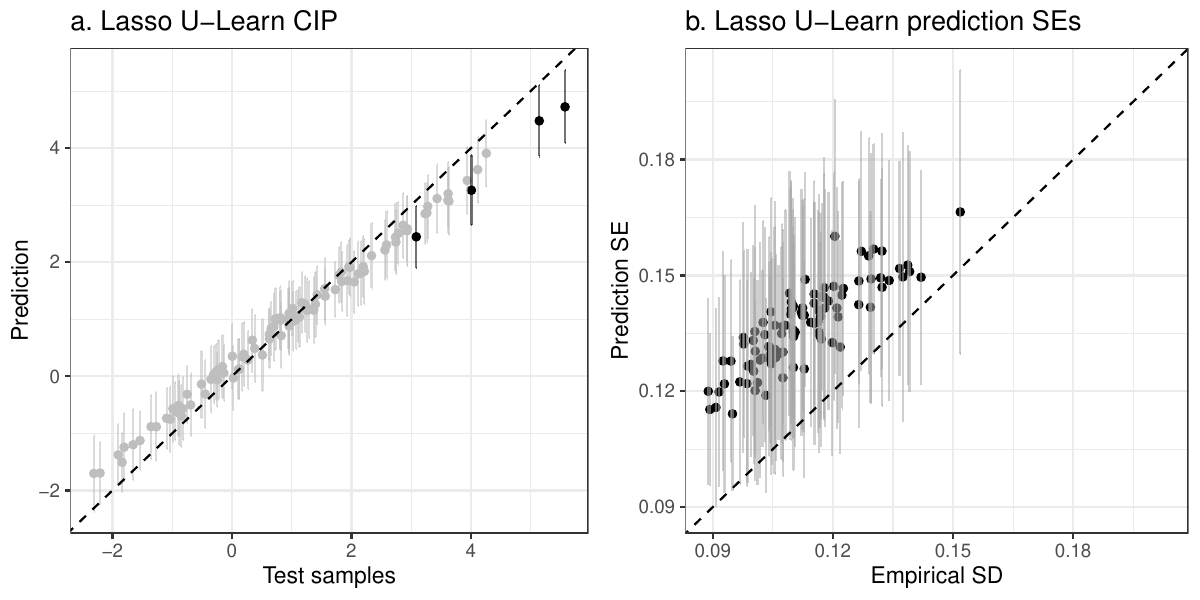}
    \includegraphics[width=\textwidth]{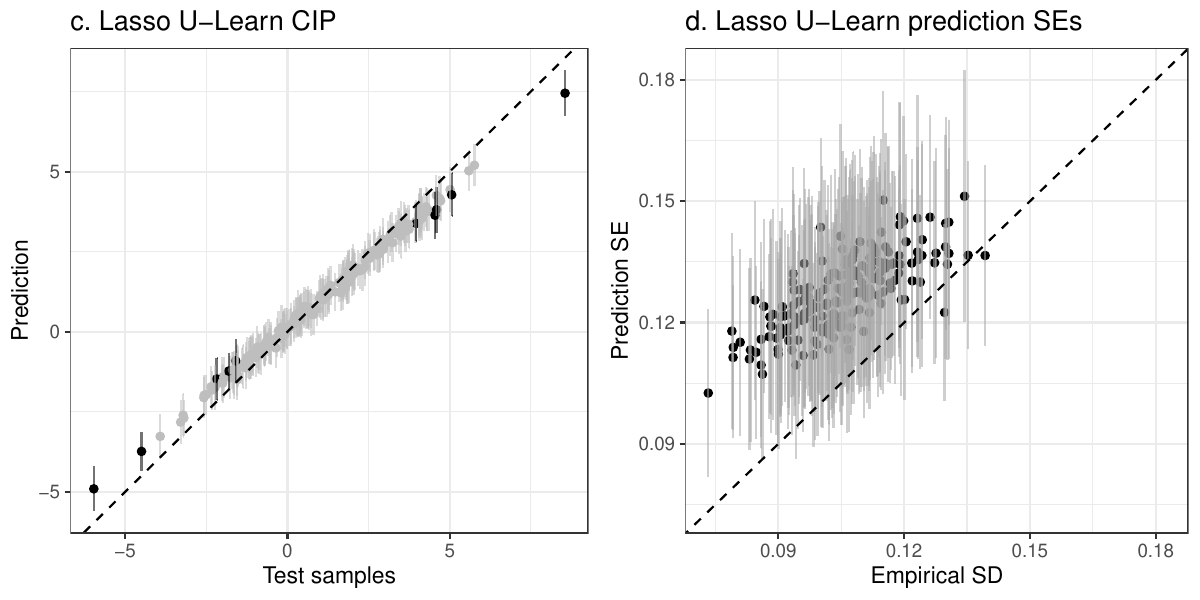}
    \caption{{\bf Prediction and inference in simulation examples 1. \bf a,b)} $n=500$; {\bf c,d)} $n=1000$. Left panels show the average CIP on the test samples; right panels show the prediction SE versus the empirical SD of all test samples. \label{fig_sim1}}
\end{figure}

{\bf Example 2 (DNN U-learning with non-linear truth). }  We considered the following scenarios of non-linear truth, 
\begin{equation*}
    \begin{aligned}
     \text{Scenario 1:}\quad  &f_0(\bx) = - \sin(\pi x_1) + 2 (x_2 - 0.5)^2 + 1.5 x_3x_4 - \frac{5}{x_5},\quad x_{j}\sim Unif(1,2);\\
     \text{Scenario 2:}\quad  &f_0(\bx) = 0.5{x_1^2} -0.3 x_5x_{10} + e^{0.2x_{15}} + \cos{x_{20}},\quad \bx \sim N(\bzero, \Sigma),
    \end{aligned}
\end{equation*}
with a total of {$p=100$} covariates and $\Sigma = \left( 0.5^{|j-k|}\right)_{j,k=1,\ldots,p}$. 
 In each scenario, we simulated 200 training sets with varying subsample sizes $r$ from $n^{0.8}$ to $n^{0.95}$, for $n=200$ and $600$, and a fixed test set of 100 samples. We applied the DNN U-learning algorithm with a network structure of $L=2$ and layer dimensions defined by $\bp = (p, 128, 64, 1)$. 
 Each hidden dense layer is also followed by a dropout layer with a 50\% dropout rate for regularization. We set $B=300$.
For comparisons, we also applied conformal prediction based on 600 training samples.
 Table \ref{tab3}  summarizes  the results obtained for each combination of $n$ and $r$. In each scenario, the $r$ values that yield the coverage probability closest to 95\% are highlighted in bold.
As $n$ increases, the mean bias, MAE, prediction SE and AIL tend to decrease. Both choices of $r$, $r=n^{0.9}$ and $r=n^{0.95}$, result in improved coverage probability (CP) and narrower AIL compared to $r=n^{0.8}$, while the computational times remain comparable across all three scenarios.
 With $n$ increasing, the prediction SE aligns more closely with the empirical SD, and both decrease gradually,  leading to more efficient CIs with coverage probability closer to the nominal level.
Again, conformal PIs are derived by using the true $f_0(\bx_*)$ in lieu of $y_*$ for a fair comparison. 
 Their prediction intervals generally exhibit wider widths and larger AILs, with coverage probabilities exceeding the nominal level,  indicating underperformance compared to our proposed method.
The U-learning method was implemented on a computing cluster with 30 CPU cores, which parallelized the subsampling scheme. This parallelization  reduces computation time compared to running the algorithm on a single core. 

\begin{table}[H]
\centering
\caption{Simulation example 2, DNN U-learning for prediction and inference. The last Time column is average computation time in seconds for one run. \label{tab3}}
\begin{tabular}{>{\rowmac}r>{\rowmac}r>{\rowmac}r>{\rowmac}r>{\rowmac}r>{\rowmac}r>{\rowmac}r>{\rowmac}r<{\clearrow}}
  \hline
 & Bias & MAE & EmpSD & SE  & CP  & AIL & Time\\  
  \hline
   \multicolumn{7}{c}{Scenario 1} \\
  $n=200,r=n^{0.8}$ &0.042 & 0.546 & 0.326 & 0.383 & 0.891 &  1.501 & 15\\ 
  \hspace{5mm} $r=n^{0.9}$ &0.033 & 0.445 & 0.259 & 0.276 & 0.906 & 1.083 & 24\\ 
  \setrow{\bfseries}\hspace{5mm} $r=n^{0.95}$ &0.022 & 0.382 & 0.222 & 0.239  & 0.935 & 0.935& 32\\ 
  $n=600,r=n^{0.8}$ &0.048 & 0.314 & 0.174 & 0.204 & 0.983 & 0.8 & 105\\ 
  \setrow{\bfseries}\hspace{5mm} $r=n^{0.9}$ & 0.030 & 0.230 & 0.170 & 0.195 & 0.954 & 0.803 & 109\\ 
  \hspace{5mm} $r=n^{0.95}$ & 0.022 & 0.189 & 0.168 & 0.184 & 0.942 & 0.840 & 135\\ 
  Conformal PI & 0.023 & 0.187 & 0.227 & - & 0.958 & 1.381 & 12 \\ 
\hline
  \multicolumn{7}{c}{Scenario 2} \\
  $n=200,r=n^{0.8}$ & 0.082 & 0.737 & 0.199 & 0.211 & 0.887 & 0.825 & 22 \\ 
  \hspace{5mm} $r=n^{0.9}$ & 0.080 & 0.520 & 0.215 & 0.221 & 0.937 & 0.866 &  24\\ 
  \setrow{\bfseries}\hspace{5mm} $r=n^{0.95}$ & 0.080 & 0.476 & 0.216 & 0.238 & 0.945 & 0.934 &  27\\ 
  $n=600,r=n^{0.8}$ & 0.084 & 0.482 & 0.148 & 0.138 & 0.881 & 0.541&  74 \\ 
  \setrow{\bfseries}\hspace{5mm} $r=n^{0.9}$ &0.078 & 0.448 & 0.145 & 0.173 & 0.945 & 0.680& 107 \\ 
  \hspace{5mm} $r=n^{0.95}$ &0.078 & 0.416 & 0.128 & 0.170 & 0.961 & 0.666 & 110 \\ 
  Conformal PI & 0.093 & 0.430 & 0.214 & - & 0.969 & 2.335 & 13 \\ 
   \hline
\end{tabular}
\end{table}

{\bf Example 3 (WHO Life Expectancy Data). } As a real data example, we analyzed a publicly available life expectancy dataset collected by WHO during 2000 to 2015 for the majority of countries in the world \citep{kaggle2023life}. 
Features in the model included important immunizations like Hepatitis B, Polio and Diphtheria and mortality factors, economic factors, social factors and other health related factors ($p=20$). Due to the dependency between the years of the same country, regression models for i.i.d. samples were not suitable, and we opted to use Algorithm \ref{alg_nn} with DNN U-learning. We randomly split the data into the training and testing sets ($n=1319$, $\ell = 330$), and applied the proposed approach with a DNN with 3 hidden layers and $\bp=(20,64,128,32,1)$. With $B=1500$, we provided the prediction and inference  in panel (a) of Figure \ref{fig_sim3}. 

The confidence intervals that did not cover the true life expectancy were highlighted in black, with an empirical coverage probability of $0.967$ in all testing samples. The average bias was $-0.12$ years and an MAE of $2.2$ years. The AIL was 16.2 years. There were a total of 9 countries whose PIs did not cover the actual life expectancies, including Bangladesh, Belarus, China, India, Mozambique, among others. Noticeably, all of them were developing countries, except for Romania (Figure \ref{fig_sim3}a).
 Notably, the country exhibiting the most substantial prediction bias, India, where the predicted life expectancy surpasses 100, also displayed a wide confidence interval, albeit not encompassing the observed life expectancy. In contrast, the conformal prediction approach yielded an average bias of $-0.98$ and Mean Absolute Error (MAE) of $2.6$ years, accompanied by confidence intervals with a consistent length of $16.9$ years across all countries. These intervals maintained an average coverage probability of $0.942$ at the population level, yet they did not capture the individual variability within each country (Figure \ref{fig_sim3}b).
It is noted that our approach outperformed the conformal prediction method in terms of covering countries with extreme life expectancies. For instance, the intervals obtained by using the proposed approach effectively covered countries like Malawi with the lowest life expectancy, as well as countries such as France, Germany, and Belgium with the highest life expectancies.

 \begin{figure}
    \centering
    \includegraphics[width=\textwidth]{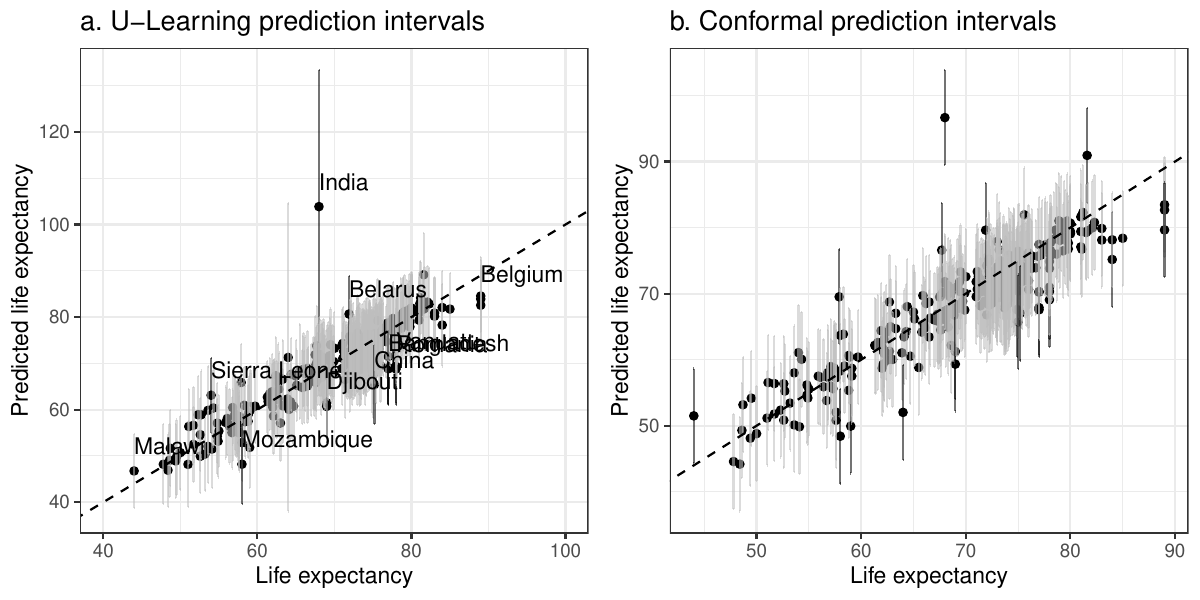}
    \caption{{\bf Prediction intervals by U-learning (left) and Conformal Prediction (right) in Example 3 life expectancy data.} Prediction intervals that do not cover the truth are in black. The labeled countries in panel (a) are (from left to right): Malawi, Sierra Leone, Mozambique, India, Djibouti, Belarus, China, Bangladesh, Romania, Vanuatu, Belgium. 
 \label{fig_sim3}}
\end{figure}

\begin{table}[ht]
\centering
\caption{Coverage of the U-learning PIs by health conditions on the test data. Fisher's exact test is performed for each health condition versus the healthy individuals.\label{tab_health}}
\begin{tabular}{rrrrrrrr}
  \hline
  & Total & \multicolumn{3}{c}{Lasso U-Learn} & \multicolumn{3}{c}{DNN U-Learn} \\
 Cover &  & 0 & 1 & Fisher's p & 0 & 1 & Fisher's p \\ 
  \hline
  Healthy &  72 &   0 &  72 &  &   0 &  72 & \\ 
\textbf{Alzheimer}  & 11& {\bfseries 4} &   {\bfseries 7} & {\boldmath $<0.001$}  &   {\bfseries 4} &   {\bfseries 7} & {\boldmath $<0.001$} \\ 
  ChronCond & 50&   3 &  47 & 0.07 &   1 &  49 & 0.41 \\ 
  HIV & 140 &   5 & 135 & 0.17 &   7 & 133 & 0.10 \\ 
  OtherBrain & 6 &   1 &   5 & 0.08 &   1 &   5 & 0.08 \\ 
  Parkinson & 24 &   1 &  23 & 0.25 &   1 &  23 & 0.25 \\ 
   \hline
\end{tabular}
\end{table}

 \section{Epigenetic Clocks with Human Methylation Data Across Multiple Tissues}

 The aging process is linked to methylation levels at specific individual CpG sites and is collectively associated with subsets of CpGs \citep{lin2016dna}. DNA methylation age, also known as epigenetic aging clocks \citep{horvath2013dna, bell2019dna, yu2020epigenetic},  employs regression models based on sets of CpG sites to predict the age of subjects at the methylation level. DNA methylation age (or simply DNA age hereafter) correlates with chronological age and holds potential for quantifying biological aging rates, as well as evaluating longevity or rejuvenating interventions \citep{marioni2015dna, lu2019dna}.

  There is a lack of uncertainty measures associated with DNA age. This challenge  motivated us to incorporate prediction and inference methods in the development of DNA age. We compiled DNA methylation data from individuals covering a wide age spectrum, ranging from 8 to 101 years old. Methylation levels were evaluated at more than 37,000 CpG sites spanning the entire human genome. Our dataset encompassed 522 blood samples and 781 non-blood samples from diverse tissue types, such as lung, liver, kidney, heart, and others. The samples comprised of healthy individuals, as well as those with various health conditions (Table \ref{tab_health}). In the following, we present comprehensive analyses with both Lasso U-learning and DNN U-learning approaches, and explore modeling the blood and non-blood samples separately or jointly.

 To assess model performance and address data heterogeneity, we started by applying Lasso U-learning separately to blood and non-blood samples.
Let $B=3000$, and the subsample size $r = n^{0.95}$. We utilized ``out-of-bag" (OOB) estimates to obtain unbiased predictions for all samples. For each subsample $b$, DNA age was predicted solely on its complement set. The performance of the two U-learning clocks concerning blood samples and non-blood samples is summarized in Figure \ref{fig_oob_DNAmAge} and Table \ref{tab:dnamage} (first two columns). Both clocks exhibit accuracy in terms of bias (-0.01 vs. 0.04) and correlation (0.99 versus 0.93) between predicted DNA age and chronological age. The notable distinction lies in the blood clock providing more accurate predictions with smaller variances, evident in the error bars and the comparison of standard errors (SE) in the Figure. Traditional aging clocks, lacking robust measures of uncertainty, might overlook such differences between homogeneous and heterogeneous samples without these new tools for prediction inference.

\begin{table}[hb]
\centering
\caption{DNA age prediction and inference in the human methylation data. Column names indicate the samples being evaluated and the model used in the parentheses.\label{tab:dnamage}}
\begin{tabular}{rrrrr}
  \hline
 & \makecell{Blood \\ (Lasso U-Learn)} & \makecell{Non-blood \\ (Lasso U-Learn)} & \makecell{Test \\(Lasso U-Learn)} & \makecell{Test \\(DNN U-Learn)} \\ 
  \hline
Bias & -0.01 & 0.04 & -0.45 & -0.44\\ 
  MAE & 2.11 & 3.65 & 3.29 & 3.0\\ 
  Corr & 0.99 & 0.93 & 0.97 & 0.97\\ 
  Mean SE & 0.84 & 2.47 & 1.70 & 1.78\\ 
  Min SE & 0.24 & 1.22 & 0.42 & 0.65\\ 
  Max SE & 3.23 & 16.87 & 10.70 & 6.58\\ 
  AIL & - & - & 0.954 & 0.954 \\
   \hline
\end{tabular}
\end{table}

 \begin{figure}
    \centering
    \includegraphics[width=\textwidth]{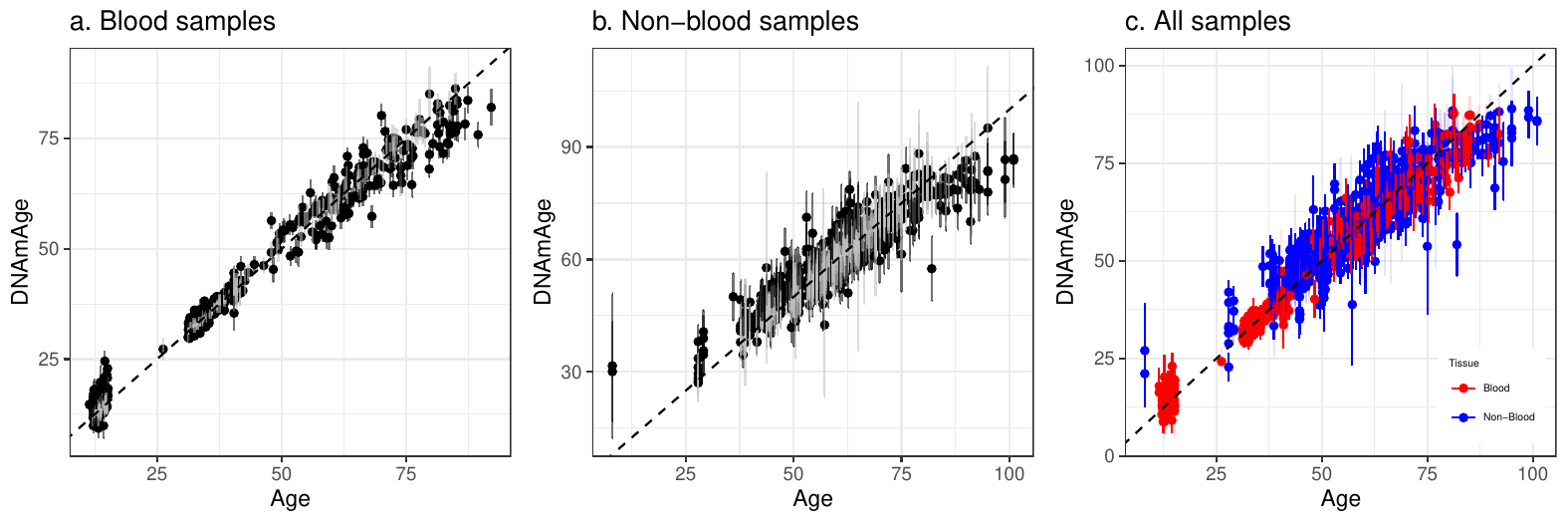}
    \caption{DNA age based on out-of-bag predictions and the prediction intervals in three clocks: a. blood samples; b. non-blood samples; c. all samples.\label{fig_oob_DNAmAge}}
\end{figure}



 To compare Lasso U-learning and DNN U-learning, we further analyzed the dataset by randomly splitting the total 1303 samples into 1000 training samples and 303 testing samples. Implementing Algorithm \ref{alg_nn}, we performed screening on the training data, reducing it to 2905 CpG sites in addition to the 13 tissue indicators. Hyperparameter search involved exploring the number of hidden layers $L$, the numbers of neurons $\bp$, and the learning rate. The final model selected was a 1-hidden-layer NN with $L = 3$, $\bp = (2918, 256, 1)$, and the hidden layer was followed by a dropout layer with a 0.5 dropout rate. 
 
  As indicated in columns 3 and 4 of Table \ref{tab:dnamage}, the DNN U-learning clock exhibited a lower Mean Absolute Error (MAE) and comparable bias to the Lasso U-learn clock. The average prediction variances were slightly larger for DNN U-learning compared to those from Lasso-based inference. The prediction variances by DNN U-learning displayed less variability than those by Lasso U-learning, yet both achieved the same coverage probability of $ACIL = 0.954$.
  
 Overall, U-learning with both Lasso and NN resulted in comparable predictions and confidence intervals (Table \ref{tab_health}, Figure \ref{fig_NN_DNAmAge}). The majority of the samples not covered were from non-blood tissues, with 2 out of 120 blood samples and 12 out of 183 non-blood samples not covered by DNN U-learning (Figure \ref{fig_NN_DNAmAge}a). Similarly, 1 blood sample and 13 non-blood samples had ages not covered by the PIs from Lasso U-learning (Figure \ref{fig_NN_DNAmAge}b). This suggested that age prediction in non-blood samples was more variable, possibly due to un-captured variability arising from the smaller number of samples per tissue.

 \begin{figure}
    \centering
    \includegraphics[width=\textwidth]{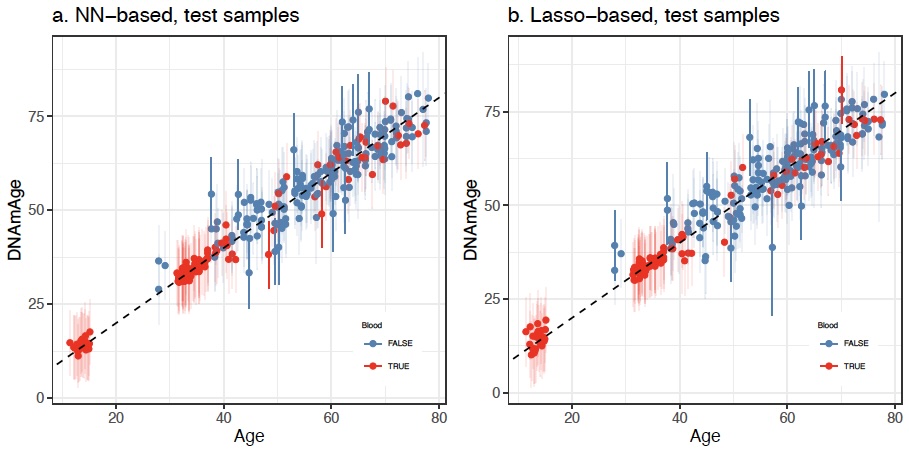}
    \includegraphics[width=\textwidth]{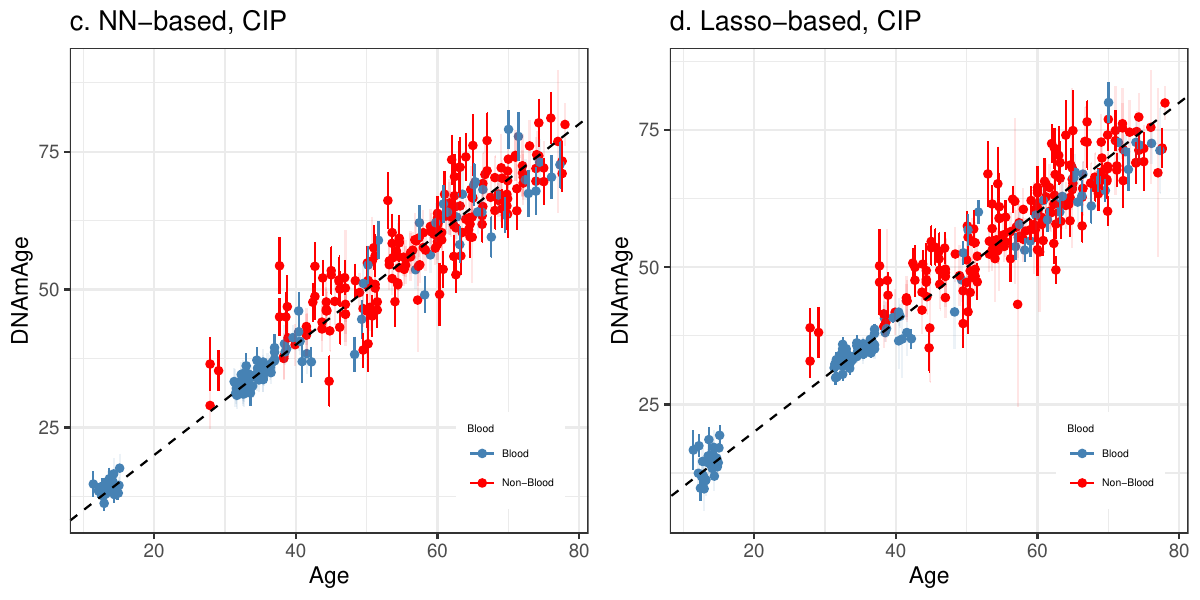}
    \caption{DNA age prediction intervals and CIP for all test samples based on two models.  \label{fig_NN_DNAmAge}} 
\end{figure}

 We also examined the U-learning PIs' coverage across various health conditions of individuals, as detailed in Table \ref{tab_health}. Patients were categorized based on their primary diagnoses, including Healthy, Alzheimer's Disease, Other Neuro-degenerative Diseases (OtherBrain), HIV, Chronic Diseases (ChronCond), and Parkinson's Disease. For all 72 healthy individuals, the PIs covered their chronological ages, affirming the accuracy of DNA ages and the reliability of inferences derived from our method. In comparison to the healthy group, the proportions of PI coverage in other disease groups were similar, except for Alzheimer's Disease, where 4 out of 11 individuals exhibit non-covered PIs. Notably, the DNA ages of these 4 individuals significantly surpassed their chronological ages, with a mean difference of 11.6 years. This outcome provided additional confirmation of the connection between DNA methylation and Alzheimer's Disease, consistent with findings reported in recent studies
\citep{mccartney2018investigating,sugden2022association,milicic2023utility}.

 In summary, our analysis offered a new perspective by pinpointing individuals whose ages notably deviated from their predicted DNA ages. Subsequent experiments may delve into the underlying biological differences among these individuals, paving the way for further exploration and understanding.


\section{Conclusions}\label{sec5}
	
We have proposed a U-learning procedure for valid prediction inference that can be applied to high-dimensional linear regressions with Lasso, as well as non-parametric machine learning algorithms like the multi-layer neural networks. The inference and confidence intervals are subject-specific that estimate the prediction variance conditional on each future sample.
The theoretical properties are derived based on generalized U-statistics and H\'ajek projections. We have illustrated the methods by comparison with conformal predictions in several numerical examples. A comprehensive real data analysis of a large scale DNA methylation study on human aging was presented, resulting in novel epigenetic clocks with valid confidence intervals that can quantify the uncertainty in predictions.

A natural extension of the proposed method is to the classification problems that involve binary or multi-class outcomes. We envision that the generalized U-statistic techniques can be updated to suit the problems.
Future research is also warranted in further improving the computational efficiency of the U-learning methods, especially in the context of neural networks.
This may be achieved by exploring the network architectures, i.e., $(L,\bp)$, and modifications to the optimization algorithms. For example, theoretical investigation of dropout layers in place of the subsampling scheme to incorporate ensemble within the training of one NN. Nevertheless, this work lays the foundation towards valid inferences with deep learning models and unveiling their black box nature.

We may further 
derive  \emph{prediction intervals}  for the unobserved $y_*$, 
which requires estimation of the error distribution; see the related literature on conformal predictions  \citep{papadopoulos2008inductive,lei2018distribution,messoudi2020conformal}.
Consider a specific extension when the errors are i.i.d. normally distributed. With $\Delta_* = y_* - \wh{f}(\bx_*)$,  note that
\begin{equation*}
\Delta_*    = \left(f_0(\bx_*) -  \wh{f}(\bx_*)\right) + \varepsilon_*
 \end{equation*}
 leads to
   \begin{equation*}\label{pred_var}
     \var(\Delta_*) = {\sigma}_*^2  + {\sigma}_\varepsilon^2,
 \end{equation*}
 as $\varepsilon_*$ is independent of $\{f_0(\bx_*) -  \wh{f}(\bx_*)\}$.
The first term on the right hand side reflects the model uncertainty, and is subject-dependent, while the second term  is the random \emph{error variance} attached to all observations. We will pursue this in future works.

\clearpage


\begin{appendix}\label{appn}

\section*{Proofs of Theorem \ref{thm_normality} and Corollary \ref{thm_var}}

We first introduce the following Lemma that gives a Lipschitz-type condition of the Lasso predictions on subsamples.

\begin{lemma}\label{lem_lip}
     Consider the Lasso predictions based on subsamples $\cT^1 = \left((\bx_1, y_1),\ldots, (\bx_r, y_r) \right)$ and $\cT^{1*} = \left((\bx_1, y_1),\ldots, (\bx_r, y_r^*) \right)$, with $y_r = f_0(\bx_r) + \varepsilon_r$, $y_r^* = f_0(\bx_r) + \varepsilon_r^*$, $\varepsilon_r$ and $\varepsilon_r^*$ are i.i.d. errors satisfying \ref{C1}. There exists a constant $c$ such that for all $r>1$,
     the following  Lipschitz-type condition is satisfied  
\begin{equation}\label{cond_lip}
    |\wt{f}\left(\bx_{r+1}; \cT^1 \right) - \wt{f}\left(\bx_{r+1}; \cT^{1*} \right)| \le c|y_{r} - y^*_{r}|,
\end{equation}
where $\wt{f}$ is defined in (\ref{lp_sub}). 
\end{lemma}

\begin{proof} 
Note that $\cT^1$ and $\cT^{1*}$ only differ in the last observation, and for notational convenience, let $\varepsilon_i^* = \varepsilon_i$ for $i <r$. Within this proof, denote $\bX = (\bx_1^\rT,\ldots, \bx_r^\rT)^\rT = (\bX_1, \ldots, \bX_p)$, $\bY = (y_1,\ldots, y_r)^\rT$, and $\by_* = (y_1,\ldots, y_r^*)^\rT$. Define the set
\[
C:= \{\beta_1\bX_1+ \ldots+ \beta_p\bX_p: \|\bbeta\|_1 \le K \}
\]
Further denote the Lasso solutions based on $\cT^1$ and $\cT^{1*}$ as $\wt{\bbeta}^{K}$ and $\wt{\bbeta}^{*K}$ respectively. By definition, $\wt{\bY} = \bX\wt{\bbeta}^{K}$ and $\wt{\bY}^* = \bX\wt{\bbeta}^{*K}$ are the projections of $\bY$ and $\by_*$ onto $C$, respectively. Since $C$ is convex, for any $\bv \in C$, $(\bv - \wt{\bY})\cdot (\bY - \wt{\bY})\le 0$. Take $\bv = \wt{\bY}^*$, we have $(\wt{\bY}^* - \wt{\bY})\cdot (\bY - \wt{\bY})\le 0$. Then
\begin{equation*}
\begin{aligned}
    \|\wt{\bY}^* - \wt{\bY}\|_2^2 \le&  (\bY - \wt{\bY}^*) \cdot (\wt{\bY} -\wt{\bY}^*) \\
    =& (\bY - \by_* + \by_* - \wt{\bY}^*) \cdot (\wt{\bY} -\wt{\bY}^*)\\
    \le& (\bY - \by_*) \cdot (\wt{\bY} -\wt{\bY}^*) \\
    =& (\varepsilon_r - \varepsilon_r^*)\bx_r (\wt{\bbeta}^{K} -\wt{\bbeta}^{*K}).
\end{aligned}
\end{equation*}
Note we used $(\by_* - \wt{\bY}^*) \cdot (\wt{\bY} -\wt{\bY}^*)\le 0$ and that $\bY$ and $\by_*$ only differed in the last observation in the above derivation. Denote $\wt{\Delta} = \wt{\bbeta}^{K} -\wt{\bbeta}^{*K}$, 
and apply the Restricted Eigenvalue (RE) condition, we have
\begin{equation*}
\begin{aligned}
    \kappa \|\wt{\Delta}\|_2^2 \le& \frac{1}{r} \|\wt{\bY}^* - \wt{\bY}\|_2^2 \\
    \le& |\varepsilon_r - \varepsilon_r^*|\frac{\|\bx_r\|_\infty}{r} \|\wt{\Delta}\|_1 \\
    \le& |\varepsilon_r - \varepsilon_r^*|\frac{\|\bx_r\|_\infty}{r} 4\sqrt{s} \|\wt{\Delta}\|_2.
    \end{aligned}
\end{equation*}
Note here $\wt{\Delta}$ is the difference between the two Lasso solutions, hence $s$ is not related to the sparsity of the true $\bbeta^0$, rather depends on $r$ and $K$ with $s = o(r)$.
Therefore,
\begin{gather*}
    \|\wt{\Delta}\|_2 = \|\wt{\bbeta}^{K} -\wt{\bbeta}^{*K}\|_2 \le \frac{4\sqrt{s}}{\kappa r}|\varepsilon_r - \varepsilon_r^*| \|\bx_r\|_\infty,
\end{gather*}
where $\kappa$ is the RE constant. Note it is a generalization of the classic $\ell_2$ error bound on the Lasso solution (for example, Theorem 7.13(b) in \cite{wainwright2019high}).
Further denote $c_{r+1}^2 = \bx_{r+1}\bx_{r+1}^\rT$, we have
\begin{equation*}
\begin{aligned}
    \text{LHS of (\ref{cond_lip})} =& \left| \bx_{r+1}\wt{\bbeta}^{K} - \bx_{r+1}\wt{\bbeta}^{*K} \right|\\
    \le& \|\bx_{r+1}\|_2 \|\wt{\bbeta}^{K} -\wt{\bbeta}^{*K}\|_2  \\
    \le& \frac{4c_{r+1}\sqrt{s}}{\kappa r}|\varepsilon_r - \varepsilon_r^*| \|\bx_r\|_\infty \\
    \le& \frac{4c_{r+1}M\sqrt{s}}{\kappa r} |\varepsilon_r - \varepsilon_r^*|.
\end{aligned}
\end{equation*}
Therefore, by choosing $c = \frac{4c_{r+1}M}{\kappa}$, (\ref{cond_lip}) holds for all $r>1$. 
\end{proof}

Before presenting the proof of Theorem \ref{thm_normality},  we introduce the H\'ajek projection \citep{van2000asymptotic,wager2018estimation}, which helps the derivation of U-statistic properties. Given a complex statistic $T$ and independent training samples $Z_1,Z_2,..,Z_n$, the H\'ajek projection of $T$ is defined to be 
	\begin{equation*}
	    \mathring{T} = \E[T] + \sum_{i=1}^{n} \left( \E[T| Z_i] - \E[T]\right),
	\end{equation*}
	i.e.,  a projection  onto a linear subspace of all random variables of the form 
 $ \sum_{i=1}^ng_i(Z_i)$, where $g_i(\cdot)$ are measurable functions.
	It follows immediately {that $\E(\mathring{T}) = E(T)$}.

 As the H\'ajek projection $\mathring{T}$ is a sum of independent random variables, we first show it is asymptotically normal under general conditions. Then, we set to bound the asymptotic difference between $T$ and $\mathring{T}$, and by using Slutsky's theorem, show that the original $T$ will be asymptotically normal as well.
    
\begin{proof}[Proof of Theorem \ref{thm_normality}]
We present the proof in the following 3 key steps.

{\bf Step 1. Asymptotic normality of the H\'ajek projection.} 
For notational ease, we first define $${T}_{n,r,B} = \wh{y}^B_* - f_0(\bx_*), $$
and write its H\'ajek projection
as
\begin{equation}\label{t_ring}
\mathring{T}_{n,r,B} =  \sum_{i=1}^{n} \E\left( {T}_{n,r,B}  | Z_i\right) = \sum_{i=1}^{n} \E\left(\wh{y}^B_* - f_0(\bx_*) | Z_i\right),
	\end{equation}
 where each term in the summation is
\begin{equation*}
\begin{aligned}
\E\left( \wh{y}^B_* - f_0(\bx_*) | Z_i\right) =&   \E\left( \frac{1}{B}\sum_{b=1}^{B} \wt{f}^{b}(\bx_*) - f_0(\bx_*) | Z_i\right) \\
=& \frac{1}{B} \sum_{b=1}^{B} \E\left( \wt{f}^{b}(\bx_*) - f_0(\bx_*) | Z_i\right).
\end{aligned}
\end{equation*} 
 Let $h_{1,r}(x) = \E\wt{f}^b(\bx_*; x, Z_2,\ldots, Z_r) - f_0(\bx_*)$ 
and $W$ be the number of subsamples that contain $i$, we have $W \sim Binom(B, \binom{n-1}{r-1}/{\binom{n}{r}} )$, and it follows
 \begin{equation*}
\begin{aligned}
\frac{1}{B} \sum_{b=1}^{B} \E\left( \wt{f}^{b}(\bx_*) - f_0(\bx_*) | Z_i\right) =&   \frac{1}{B} \sum_{b=1}^{B} \E\left( \E\big(  \wt{f}^{b}(\bx_*) - f_0(\bx_*) | Z_i\big) | W \right) \\
=& \frac{1}{B} \E\left(W h_{1,r}(Z_i) \right) \\
=& \frac{1}{B} \left(B \binom{n-1}{r-1}/{\binom{n}{r}} h_{1,r}(Z_i) \right)\\
=& \frac{r}{n} h_{1,r}(Z_i).
\end{aligned}
\end{equation*} 
Plug this back into (\ref{t_ring}) yields
\begin{equation*}
\mathring{T}_{n,r,B} = \frac{r}{n} \sum_{i=1}^{n} h_{1,r}(Z_i)  .
\end{equation*}
To show the asymptotic normality of $\mathring{T}_{n,r,B}$, we can define the triangular array for $ r_nh_{1,r_n}(Z_i)$'s and verify the Lindeberg condition  as below.
For any $\delta > 0$, 
\begin{align*}
&  \sum_{i=1}^{n} \frac{1}{n r^2_{n} \xi_{1,r_n}} \int_{|r_{n}h_{1,r_n}(Z_i)|\geq\delta\sqrt{n\xi_{1,r_n}}} r^2_{n} h^2_{1,r_n}(Z_i)dP \\
&=  \sum_{i=1}^{n} \frac{1}{n \xi_{1,r_n}} \int_{|h_{1,r_n}(Z_i)|\geq\delta\sqrt{n\xi_{1,r_n}}} h^2_{1,r_n}(Z_i)dP \\
&=  \frac{1}{\xi_{1,r_n}} \int_{|h_{1,r_n}(Z_1)|\geq\delta\sqrt{n\xi_{1,r_n}}} h^2_{1,r_n}(Z_1)dP. 
\end{align*}
The above Lindeberg condition is satisfied by Lemma \ref{lem_lip}.
To see it, we repeatedly apply (\ref{cond_lip}) and Jensen's inequality to get the following bound,
\begin{align*}
& \sup_{\bx_1 \in \mathcal{X}} \left| h_{1,r_n}\left( (\bx_1, y_1^*) \right) \right| \\
= & \sup_{\bx_1 \in \mathcal{X}} \Bigg| \int \Big( \wt{f}\big( (\bx_1, y_1^*), (x_2, y_2), \ldots, (x_{r_n}, y_{r_n}) \big) - \wt{f}\big( (\bx_1, y_1^*), \ldots, (x_{r_n}, y_{r_n}^*) \big) \Big) \\
& \quad + \wt{f}\big( (\bx_1, y_1^*), \ldots, (x_{r_n}, y_{r_n}^*) \big) \, dP - f_0 \Bigg| \\
\leq & \sup_{\bx_1 \in \mathcal{X}} \int \Big( \wt{f}\big( (\bx_1, y_1^*), (x_2, y_2), \ldots, (x_{r_n}, y_{r_n}) \big) - \wt{f}\big( (\bx_1, y_1^*), \ldots, (x_{r_n}, y_{r_n}^*) \big) \Big) \, dP \\
& + \sup_{\bx_1 \in \mathcal{X}} \wt{f}\big( (\bx_1, y_1^*), \ldots, (x_{r_n}, y_{r_n}^*) \big) - f_0 \\
\leq & c r_n \mathbb{E} \left| \epsilon_1 \right| + M - f_0.
\end{align*}
Then
\begin{align*}
A_n &= \left\{ \left| h_{1,r_n}\left( (\bx_1, y_1) \right) \right| \geq \delta \sqrt{n \zeta_{1,r_n}} \right\} \\
   &= \left\{ \left| h_{1,r_n}\left( (\bx_1, y_1) \right) - h_{1,r_n}\left( (\bx_1, y_1^*) \right) + h_{1,r_n}\left( (\bx_1, y_1^*) \right) \right| \geq \delta \sqrt{n \zeta_{1,r_n}} \right\} \\
   &\subseteq \left\{ \left| h_{1,r_n}\left( (\bx_1, y_1) \right) - h_{1,r_n}\left( (\bx_1, y_1^*) \right) \right| \geq \delta \sqrt{n \zeta_{1,r_n}} - \left| h_{1,r_n}\left( (\bx_1, y_1^*) \right) \right| \right\} \\
   &\subseteq \left\{ \left| \epsilon_1 \right| \geq \frac{1}{c} \left( \delta \sqrt{n \zeta_{1,r_n}} + M - f_0 \right) + r_n \mathbb{E} \left| \epsilon_1 \right| \right\} \\
   &:= A_n^*.
\end{align*}
In the above derivation, we used the bound
\[
\sup_{\bx_1\in \cX} \left| h_{1,r_n}\left( (\bx_1, y_1^*) \right) \right| \le cr_n \mathbb{E} \left| \epsilon_1 \right| + M - f_0.
\]
Further we have
\begin{align*}
& \frac{1}{\xi_{1,r_n}} \int_{A_n} h^2_{1,r_n}\left((\bx_1, y_1) \right)dP \\
&=  \frac{1}{\xi_{1,r_n}} \int_{A_n} \left(h_{1,r_n}\big((\bx_1, y_1)\big) - h_{1,r_n}\big((\bx_1, y_1^*)\big) + h_{1,r_n}\big((\bx_1, y_1^*)\big) \right)^2 dP \\
&\leq \frac{2}{\xi_{1,r_n}}  \int_{A_n} (h_{1,r_n}((\bx_1, y_1)) - h_{1,r_n}((\bx_1, y_1^*)))^2 dP \\
&\quad   +  \frac{2}{\xi_{1,r_n}} \int_{A_n } h^2_{1,r_n}((\bx_1, y_1^*))dP \\
&\leq \frac{2}{\xi_{1,r_n}} \int_{A_n^*} c^2 \varepsilon_1^2 dP +  \frac{2}{\xi_{1,r_n}} P(A^*_n) (c r_n \E|\varepsilon_1| + M - f_0)^2 \\
& =  \ \frac{2}{\xi_{1,r_n}} P\left[|\varepsilon_1| \geq \frac{1}{c}(\sqrt{n \xi_{1,r_n}} + M - f_0) +  \E|\varepsilon_1| \right] \times (c r_n E|\varepsilon_1| + M - f_0)^2 \\
& \rightarrow 0.
\end{align*}  
Here the second inequality is by (\ref{cond_lip}), and the last equality is due to the independence between $A^*_n$ and $\varepsilon_1$.
Therefore the Lindeberg condition is satisfied, and by the Lindeberg-Feller central limit theorem,
\begin{equation}\label{T_ring}
    \frac{\sqrt{n} \mathring{T}_{n,r,B} }{\sqrt{r_n^2 \xi_{1,r_n}} }\xrightarrow{d} N(0,1).
\end{equation}


{\bf Step 2. Asymptotic normality of $T_{n,r,B}$.} The base learner $\wt{T} = \wt{f}(\bx_*) - f_0(\bx_*)$ is symmetric in the subsamples $Z_1,Z_2,\ldots, Z_r$.
The Efron-Stein ANOVA decomposition \citep{efron1981jackknife} states that there exist functions $T_1, T_2,.., T_r$ such that
	\begin{equation*}
	    \wt{T}(Z_1,Z_2,..,Z_r) = \sum_{i=1}^{r}T_1(Z_i) + \sum_{i<j} T_2(Z_i, Z_j) + ... +
	    T_r(Z_1,Z_2,..,Z_r),
	\end{equation*}
	and that all $2^r - 1$ random variables on the right hand side are all mean-zero and uncorrelated.
Applying the ANOVA decomposition to each base learner in $T_{n,r,B}$ as in
$T_{n,r,B} = \wh{y}^B_* - f_0(\bx_*) =\frac{1}{B} \sum_{b=1}^{B}\wt{T}^b$ yields
\begin{gather*}
    T_{n,r,B} = \frac{1}{B}\left( \binom{n-1}{r-1} \sum_{i=1}^{n}T_1(Z_i) + \binom{n-2}{r-2} \sum_{i<j} T_2(Z_i, Z_j) + \right.\\
  \left.  \ldots + \sum_{i_1<\ldots<i_r} T_r(Z_{i_1},\ldots,Z_{i_r}) \right)
\end{gather*}
As with all projections,
\begin{align*}
\E \left[ \left({T}_{n,r,B} - \mathring{T}_{n,r,B}\right)^2 \right] &= \var\left[ {T}_{n,r,B} - \mathring{T}_{n,r,B}\right] \\
&= \sum_{k=2}^{r} \left(\frac{r_k}{n_k} \right)^2 \binom{n}{k} V_k, \\
&= \sum_{k=2}^{r}\frac{r_k}{n_k} \binom{r}{k} V_k, \\
&\leq \frac{r_2}{n_2} \sum_{k=2}^{r} \binom{r}{k} V_k, \\
&\leq \frac{r_2}{n_2} \var[\wt{T}],
\end{align*}
where $V_k = \var(T_k)$ and $r_k = r(r-1)\cdots (r-k)$. Lastly, $r_2/n_2\le r^2/n^2$.

Write $\var(\mathring{T}_{n,r,B}) = v_n^2/n = r_n^2\xi_{1,r_n}/n$, we have
\begin{align}
    \frac{1}{\var(\mathring{T}_{n,r,B})} \E \left[ \left({T}_{n,r,B} - \mathring{T}_{n,r,B}\right)^2 \right] &\le \left(\frac{r}{n} \right)^2 \frac{\var[\wt{T}]}{r_n^2\xi_{1,r_n}/n} \\
    &= \frac{1}{n} \var[\wt{T}]/ \xi_{1,r_n}\\
    &\rightarrow 0. \label{eq_diff}
\end{align}
Apply Slutsky's theorem to (\ref{T_ring}) and (\ref{eq_diff}), we get
\begin{equation*}
    \frac{\sqrt{n} {T}_{n,r,B}}{\sqrt{r_n^2 \xi_{1,r_n}} }\xrightarrow{d} N(0,1).
\end{equation*}

{\bf Step 3. Consistency of the variance estimator.}
We are only left to show $ n\wh{V}^B/ v_n^2 \stackrel{p}{\rightarrow} 1$, or equivalently  $ \wh{V}^B/ \var(\mathring{T}_{n,r,B}) \stackrel{p}{\rightarrow} 1$.

From its definition, $\var(\mathring{T}_{n,r,B})$ can be written as
	\begin{equation}\label{sigma2}
	    \begin{aligned}
	        \delta_n^2  &= \sum_{i=1}^{n}\left(\E[\wh{y}_*^B | Z_i] - \E[\wh{y}_*^B] \right)^2 \\
	        &= \frac{r^2}{n^2}\sum_{i=1}^{n}\left(\E[\wt{y}_*^b | Z_i] - \E[\wt{y}_*^b] \right)^2. 
	    \end{aligned}
	\end{equation}
 This is because $\E[\wt{y}_*^b | Z_i] - \E[\wt{y}_*^b] = 0$ for those subsamples where $Z_i$ is not sampled.
	On the other hand, we write the infinitesimal jackknife estimate as
	\begin{equation} \label{V_inf}
	    \wh{V}^B = \frac{n-1}{n} \left(\frac{n}{n-r} \right)^2 \frac{r^2}{n^2}\sum_{i=1}^{n}\left(\E^*[\wt{y}_*^b | Z_1^* = Z_i] - \E^*[\wt{y}_*^b] \right)^2.
	\end{equation}
	It is worth noting that that $\E$ is with respect to the true (theoretical) distribution that generates  $Z_1, Z_2, \ldots,Z_n$, whereas $\E^*$ is with respect to the empirical distribution based on $Z_1, Z_2, \ldots,Z_n$.
	
	Equation (\ref{V_inf}) can further be decomposed using the H\'ajek projection $\mathring{y}^*$ of $\wt{y}_*^b$:
	\begin{equation}\label{eq_AR}
	    \wh{V}^B = \frac{n-1}{n} \left(\frac{n}{n-r} \right)^2 \frac{r^2}{n^2}\sum_{i=1}^{n}\left(A_i + R_i \right)^2,
	\end{equation}
	where
	\begin{gather*}
	    A_i = \E^*[\mathring{y}^*| Z_1^* = Z_i] - \E^*[\mathring{y}^*], \\
	    R_i = \E^*[\wt{y}_*^b - \mathring{y}^* | Z_1^* = Z_i] - \E^*[\wt{y}_*^b - \mathring{y}^*].
	\end{gather*}
	We will show the main effects $A_i$'s give us $\delta_n^2$ and the sum of $R_i$'s goes to zero. {Therefore, $\lim_{n\rightarrow \infty} \wh{V}^B/ \var[\mathring{T}_{n,r,B}] = 1$.} 
To show it, we write 
	\begin{equation*}
	    \begin{aligned}
	        A_i &= \E^*[\mathring{y}^*| Z_1^* = Z_i] - \E^*[\mathring{y}^*] \\
	        &= \left( 1 - \frac{r}{n}\right) T_1(Z_i) + \left( \frac{r-1}{n-1} - \frac{r}{n}\right) \sum_{j\ne i} T_1(Z_j),
	    \end{aligned}
	\end{equation*}
    where $T_1(Z_i) = \E[\wt{y}_*^b | Z_i] - \E[\wt{y}_*^b]$, and further
        \begin{equation*} 
         \begin{aligned}
	    \E\left[ \frac{n-1}{n} \left(\frac{n}{n-r} \right)^2 \frac{r^2}{n^2}\sum_{i=1}^{n}A_i^2 \right] =& \frac{r^2}{n^2} \sum_{i=1}^{n}  T_1(Z_i)^2\\
     =& \delta_n^2.
        \end{aligned}
	\end{equation*}
 From the above calculation we can also get $\delta_n^{-2}\frac{r^2}{n^2}\sum_{i=1}^n(A_i - T_1(Z_i))^2 \rightarrow_p 0$. Thus by the weak law of large numbers for triangular arrays for $ \frac{1}{\delta_n^2} \frac{r^2}{n^2}\sum_{i=1}^{n}T_1(Z_i)^2$, we have
    \begin{equation*} 
	  \frac{1}{\delta_n^2} \frac{r^2}{n^2}\sum_{i=1}^{n}A_i^2  \rightarrow_p 1.
	\end{equation*}

 On the other hand, we can get the following bound following Lemma 13 of \cite{wager2018estimation},
 \begin{equation*}
     \E [R_i^2] \lesssim \frac{2}{n}\var\left[\wt{y}_*^b \right],
 \end{equation*}
 which leads to
  \begin{equation*}
   \begin{aligned}
     \E \left[\frac{r^2}{n^2} \sum R_i^2\right] \lesssim&
     \frac{2r^2}{n^2}\var\left[\wt{y}_*^b \right] \\
     \lesssim& \frac{2r}{n } \delta_n^2.
      \end{aligned}
 \end{equation*}
 The last step used the decomposition (\ref{sigma2}) and $\delta_n^2 = \frac{r}{n} \var\left[\wt{y}_*^b \right]$. 
 Then by Markov's inequality, as $r/n \rightarrow 0$,
 $$
 \frac{1}{\delta_n^2}\frac{r^2}{n^2} \sum R_i^2 \xrightarrow{p} 0.
 $$
 Lastly, applying $\left(A_i + R_i \right)^2 \le 2(A_i^2 + R_i^2)$ {to (\ref{eq_AR})}, we conclude that $ \wh{V}^B/\delta_n^2$ converges in probability to 1.

\end{proof}

\begin{proof}[Proof of Corollary \ref{thm_var}]
With the results of Step 3 in the proof of Theorem \ref{thm_normality} and by using the Slutsky theorem, we have that
    \begin{gather*}
        \frac{ \wh{y}_*^B - f_0(\bx_*)}{\wh{\sigma}_* }  \xrightarrow{d} N(0,1).
    \end{gather*}
    Therefore, 
    \begin{gather*}
       \Pr\left[ f_0(\bx_*) \in \left( \wh{L}(\bx_*), \wh{U}(\bx_*) \right) \right] = \Pr\left[ \left|\frac{ f_0(\bx_*) - \wh{y}^B_*}{\wh{\sigma}_* } \right| < z_{1-\alpha/2} \right] \\
       \rightarrow 
       \Pr [|N(0,1)|<  z_{1-\alpha/2}]=1 - \alpha.
    \end{gather*}
    
\end{proof}

\section*{Proof of DNN U-learning Theorem \ref{thm_NN}}
{The proof of Theorem \ref{thm_NN} uses arguments for incomplete generalized U-statistics with random kernels, i.e., each $\wt{f}^{b_j}(\bx)$ is not deterministic given the subsample of training data. The key is that the randomness in $\wt{f}^{b_j}(\bx)$ is independent of the original samples. Our proof takes advantage of the U-statistic property that the asymptotic normality does not dependent on the exact form of the kernel, which is the NN fit in this case. }

\begin{proof}[Proof of Theorem \ref{thm_NN} ] 
 Denote $\omega_j$ as the randomization parameters (SGD, random dropout, etc) involved in fitting the NN with the $j$-th subsample in Algorithm \ref{alg_nn}.
 \begin{gather*}
     \wt{f}^{b_j}(\bx) = \wt{f}^{(\omega_j)}\left( \bx|L, \bp, \cT^{b_j} \right) \\
     \wh{y}^B_* = \frac{1}{B}\sum_{j} \wt{f}^{b_j}(\bx_*) = \frac{1}{B}\sum_{j} \wt{f}^{(\omega_j)}\left( \bx_*|L, \bp,\cT^{b_j} \right).
 \end{gather*}
 Suppose $\omega_1, \ldots, \omega_B \sim_{i.i.d.} F_\omega$, and they are independent of the observations $\cT_n$ and the random subsampling. Consider the statistic \(  \wh{y}^{*B}_* = \E_\omega \wh{y}^B_* \)
 by taking the expectation with respect to $\omega$.  Then \( \wh{y}^{*B}_*\) is an incomplete generalized U-statistic by definition. 

In order to apply Theorem \ref{thm_normality} to  $\wh{y}^{*B}_*$, we verify the required conditions as follows. First, \ref{D1} is analogous to \ref{C1}, \ref{D2} and \ref{C2} are identical, \ref{D3} and \ref{D4} together bound the tail behavior of the error term and the predictions $\wt{f}^{b_j}(\bx_*)$'s. By Lemma \ref{lem_nn}, 
\[
R(\wt{f}^{b_j}, f_0 ) \le  C'\phi_{n} L \log^2 r_n.
\]
Equivalently, 
\[
\E\left[ \left( \wt{f}^{b_j}(\bx_*) - f_0(\bx_*) \right)^2 \right] \le  C'\phi_{n} L \log^2 r_n.
\]
Then
\begin{align*}
    \E\left( \wh{y}^{*B}_* - f_0(\bx_*) \right)^2 &= 
    \E\left( \E_\omega \wh{y}^B_* - f_0(\bx_*) \right)^2\\
    &= \E \left[ \E_\omega \left(  \wh{y}^B_* - f_0(\bx_*) \right)^2 \right] \\
    &= \E \left[ \E_\omega \left(  \frac{1}{B}\sum_{j=1}^B \wt{f}^{b_j}(\bx_*) - f_0(\bx_*) \right)^2 \right] \\
    &\le \E \left[ \E_\omega \frac{1}{B}  \sum_{j=1}^B \left( \wt{f}^{b_j}(\bx_*) - f_0(\bx_*) \right)^2 \right] \\
    &\le  C'\phi_{n} L \log^2 r_n \\
    &\lesssim  \phi_{n}(\log^{\alpha}n)  (\log^2 r_n)   \\
      &\lesssim  \phi_{n} \log^{2+\alpha} n,
\end{align*}
which goes to zero 
by Lemma \ref{lem_nn}. Here, the last inequality is because of  \ref{D2}.
 Therefore, $\E \wh{y}^{*B}_* \rightarrow f_0(\bx_*)$. Let 
 \[
\xi_{1,r_n}(\bx_*) = \cov\left(\wt{f}^{(\omega)}(\bx_*; Z_1,Z_2,..,Z_{r_n}), \wt{f}^{(\omega)}(\bx_*; Z_1,Z_2'..,Z_{r_n}') \right),
\]
where the covariance is taken over $\omega$ as well.
Since $\wt{f}^{(\omega)} \in \cF(L,\bp, s,F)$ with a uniform upper bound $F$, $\E \{\wt{f}^{(\omega)}(\bx)\}^2 \le F^2 < \infty$. Hence, the non-random version \(  \wh{y}^{*B}_* = \E_\omega \wh{y}^B_* \) satisfies all the conditions of Theorem \ref{thm_normality}, and we can obtain
 \begin{equation*}
	    \frac{\sqrt{n}\left(\wh{y}^{*B}_* - f_0(\bx_*) \right) }{ v_{n} (\bx_*) } \xrightarrow{d} N(0,1),
	\end{equation*}
with $v_n(\bx_*) = \sqrt{r_n^2 \xi_{1,r_n}(\bx_*) }$. 
 Therefore, in order to show the asymptotic normality for $\wh{y}^B_*$, we are left to show 
 \begin{equation}\label{eq_Eomega}
    \frac{ \sqrt{n}\left( \wh{y}^{*B}_* -  \wh{y}^B_*  \right)}{v_n(\bx_*) } \xrightarrow{P} 0.
 \end{equation}
 
To proceed, we first  write out the following expression, and derive 
\begin{equation*}
    \begin{aligned}
        \E \left( \wh{y}^{B}_* -  \wh{y}^{*B}_*  \right)^2 =& \E \left[\frac{1}{B}\sum_{j} \wt{f}^{(\omega_j)}\left( \bx_*|\cT^{b_j} \right) - \E_\omega\left(\frac{1}{B}\sum_{j} \wt{f}^{(\omega_j)}\left( \bx_*|\cT^{b_j} \right) \right) \right]^2 \\
      =& \E  \left[ \frac{1}{B^2} \sum_{j} \left( \wt{f}^{(\omega_j)}\left( \bx_*|\cT^{b_j} \right) - \E_\omega \wt{f}^{(\omega_j)}\left( \bx_*|\cT^{b_j} \right) \right)^2 \right]\\ 
  =& \frac{1}{B^2} \sum_{j}\E \left( \wt{f}^{(\omega_j)}\left( \bx_*|\cT^{b_j} \right) - \E_\omega \wt{f}^{(\omega_j)}\left( \bx_*|\cT^{b_j} \right) \right)^2\\
   =& \frac{1}{B} \E \left( \wt{f}^{(\omega_j)}\left( \bx_*|\cT^{b_j} \right) - \E_\omega \wt{f}^{(\omega_j)}\left( \bx_*|\cT^{b_j} \right) \right)^2.
    \end{aligned}
\end{equation*}

Then for any $\epsilon_0>0$, applying the Chebyshev inequality gives 
\[
\Pr \left( \left | \frac{\sqrt{n}\left( \wh{y}^{*B}_* -  \wh{y}^B_*  \right)}{ v_n(\bx_*) }\right | > \epsilon_0\right)
\le \frac{1}{\epsilon_0^2   v^2_n(\bx_*) }  \E \{n\left( \wh{y}^{*B}_* -  \wh{y}^B_*  \right)^2\}.
\]

On the other hand,  we have
\begin{equation*}
    \begin{aligned}
        &  \frac{1}{ \epsilon_0^2 v^2_n(\bx_*) }  \E \{n\left( \wh{y}^{*B}_* -  \wh{y}^B_*  \right)^2\}  \\
        =& \frac{1}{\epsilon_0^2}  \frac{1}{r_n^2 \xi_{1,r_n}(\bx_*)} \frac{n}{B}
        \E  \left( \wt{f}^{(\omega_j)}\left( \bx_*|\cT^{b_j} \right) - \E_\omega \wt{f}^{(\omega_j)}\left( \bx_*|\cT^{b_j} \right) \right)^2 \\
         \rightarrow & 0,
    \end{aligned}
\end{equation*}
which holds because, by assumptions,  $n/B \rightarrow 0$, $\liminf \xi_{1,r_n}(\bx_*) > 0$,  and  \[
\lim_{n\rightarrow \infty} \E  \left( \wt{f}^{(\omega_j)}\left( \bx_*|\cT^{b_j} \right) - \E_\omega \wt{f}^{(\omega_j)}\left( \bx_*|\cT^{b_j} \right) \right)^2 < \infty.
\] 
Therefore,  
$
\Pr \left( \left | \frac{\sqrt{n}\left( \wh{y}^{*B}_* -  \wh{y}^B_*  \right)}{ v_n(\bx_*) }\right | > \epsilon_0\right)
 \rightarrow 0$ for any $\epsilon_0>0$. Hence (\ref{eq_Eomega}) holds, and the proof is completed by applying the Slutzky theorem.


 \end{proof}

\end{appendix}



%
%

\begin{acks}[Acknowledgments]
This work is partially supported by NIH grants. We are grateful toward Dr. Steve  Horvath for providing the DNA Methylation data from the Mammalian Methylation Consortium and insightful discussions.
\end{acks}
%


\bibliographystyle{imsart-number} 
\bibliography{mybib}      

\begin{thebibliography}{68}

\bibitem{andreassen2020asymptotics}
\begin{barticle}[author]
\bauthor{\bsnm{Andreassen},~\bfnm{Anders}\binits{A.}} \AND
  \bauthor{\bsnm{Dyer},~\bfnm{Ethan}\binits{E.}}
(\byear{2020}).
\btitle{Asymptotics of wide convolutional neural networks}.
\bjournal{arXiv preprint arXiv:2008.08675}.
\end{barticle}
\endbibitem

\bibitem{angelopoulos2021gentle}
\begin{barticle}[author]
\bauthor{\bsnm{Angelopoulos},~\bfnm{Anastasios~N}\binits{A.~N.}} \AND
  \bauthor{\bsnm{Bates},~\bfnm{Stephen}\binits{S.}}
(\byear{2021}).
\btitle{A gentle introduction to conformal prediction and distribution-free
  uncertainty quantification}.
\bjournal{arXiv preprint arXiv:2107.07511}.
\end{barticle}
\endbibitem

\bibitem{athey2018approximate}
\begin{barticle}[author]
\bauthor{\bsnm{Athey},~\bfnm{Susan}\binits{S.}},
  \bauthor{\bsnm{Imbens},~\bfnm{Guido~W}\binits{G.~W.}} \AND
  \bauthor{\bsnm{Wager},~\bfnm{Stefan}\binits{S.}}
(\byear{2018}).
\btitle{Approximate residual balancing: debiased inference of average treatment
  effects in high dimensions}.
\bjournal{Journal of the Royal Statistical Society: Series B (Statistical
  Methodology)}
\bvolume{80}
\bpages{597--623}.
\end{barticle}
\endbibitem

\bibitem{avagyan2021high}
\begin{barticle}[author]
\bauthor{\bsnm{Avagyan},~\bfnm{Vahe}\binits{V.}} \AND
  \bauthor{\bsnm{Vansteelandt},~\bfnm{Stijn}\binits{S.}}
(\byear{2022}).
\btitle{High-dimensional inference for the average treatment effect under model
  misspecification using penalized bias-reduced double-robust estimation}.
\bjournal{Biostatistics \& Epidemiology}
\bvolume{6}
\bpages{221--238}.
\end{barticle}
\endbibitem

\bibitem{barber2021predictive}
\begin{barticle}[author]
\bauthor{\bsnm{Barber},~\bfnm{Rina~Foygel}\binits{R.~F.}},
  \bauthor{\bsnm{Cand{\`e}s},~\bfnm{Emmanuel~J.}\binits{E.~J.}},
  \bauthor{\bsnm{Ramdas},~\bfnm{Aaditya}\binits{A.}} \AND
  \bauthor{\bsnm{Tibshirani},~\bfnm{Ryan~J.}\binits{R.~J.}}
(\byear{2021}).
\btitle{Predictive inference with the jackknife+}.
\bjournal{Ann. Statist.}
\bvolume{49}
\bpages{486--507}.
\bdoi{10.1214/20-AOS1965}
\end{barticle}
\endbibitem

\bibitem{bartlett2012l}
\begin{barticle}[author]
\bauthor{\bsnm{Bartlett},~\bfnm{Peter~L}\binits{P.~L.}},
  \bauthor{\bsnm{Mendelson},~\bfnm{Shahar}\binits{S.}} \AND
  \bauthor{\bsnm{Neeman},~\bfnm{Joseph}\binits{J.}}
(\byear{2012}).
\btitle{$\ell$1 - regularized linear regression: persistence and oracle
  inequalities}.
\bjournal{Probability theory and related fields}
\bvolume{154}
\bpages{193--224}.
\end{barticle}
\endbibitem

\bibitem{bauer2019deep}
\begin{barticle}[author]
\bauthor{\bsnm{Bauer},~\bfnm{Benedikt}\binits{B.}} \AND
  \bauthor{\bsnm{Kohler},~\bfnm{Michael}\binits{M.}}
(\byear{2019}).
\btitle{On deep learning as a remedy for the curse of dimensionality in
  nonparametric regression}.
\bjournal{The Annals of Statistics}
\bvolume{47}
\bpages{2261--2285}.
\bdoi{10.1214/18-AOS1747}
\end{barticle}
\endbibitem

\bibitem{bell2019dna}
\begin{barticle}[author]
\bauthor{\bsnm{Bell},~\bfnm{Christopher~G}\binits{C.~G.}},
  \bauthor{\bsnm{Lowe},~\bfnm{Robert}\binits{R.}},
  \bauthor{\bsnm{Adams},~\bfnm{Peter~D}\binits{P.~D.}},
  \bauthor{\bsnm{Baccarelli},~\bfnm{Andrea~A}\binits{A.~A.}},
  \bauthor{\bsnm{Beck},~\bfnm{Stephan}\binits{S.}},
  \bauthor{\bsnm{Bell},~\bfnm{Jordana~T}\binits{J.~T.}},
  \bauthor{\bsnm{Christensen},~\bfnm{Brock~C}\binits{B.~C.}},
  \bauthor{\bsnm{Gladyshev},~\bfnm{Vadim~N}\binits{V.~N.}},
  \bauthor{\bsnm{Heijmans},~\bfnm{Bastiaan~T}\binits{B.~T.}},
  \bauthor{\bsnm{Horvath},~\bfnm{Steve}\binits{S.}} \betal{et~al.}
(\byear{2019}).
\btitle{DNA methylation aging clocks: challenges and recommendations}.
\bjournal{Genome biology}
\bvolume{20}
\bpages{1--24}.
\end{barticle}
\endbibitem

\bibitem{belloni2014high}
\begin{barticle}[author]
\bauthor{\bsnm{Belloni},~\bfnm{Alexandre}\binits{A.}},
  \bauthor{\bsnm{Chernozhukov},~\bfnm{Victor}\binits{V.}} \AND
  \bauthor{\bsnm{Hansen},~\bfnm{Christian}\binits{C.}}
(\byear{2014}).
\btitle{High-dimensional methods and inference on structural and treatment
  effects}.
\bjournal{Journal of Economic Perspectives}
\bvolume{28}
\bpages{29--50}.
\end{barticle}
\endbibitem

\bibitem{belloni2014inference}
\begin{barticle}[author]
\bauthor{\bsnm{Belloni},~\bfnm{Alexandre}\binits{A.}},
  \bauthor{\bsnm{Chernozhukov},~\bfnm{Victor}\binits{V.}} \AND
  \bauthor{\bsnm{Hansen},~\bfnm{Christian}\binits{C.}}
(\byear{2014}).
\btitle{Inference on treatment effects after selection among high-dimensional
  controls}.
\bjournal{The Review of Economic Studies}
\bvolume{81}
\bpages{608--650}.
\end{barticle}
\endbibitem

\bibitem{belloni2016post}
\begin{barticle}[author]
\bauthor{\bsnm{Belloni},~\bfnm{Alexandre}\binits{A.}},
  \bauthor{\bsnm{Chernozhukov},~\bfnm{Victor}\binits{V.}} \AND
  \bauthor{\bsnm{Wei},~\bfnm{Ying}\binits{Y.}}
(\byear{2016}).
\btitle{Post-selection inference for generalized linear models with many
  controls}.
\bjournal{Journal of Business \& Economic Statistics}
\bvolume{34}
\bpages{606--619}.
\end{barticle}
\endbibitem

\bibitem{bolcskei2019optimal}
\begin{barticle}[author]
\bauthor{\bsnm{Bolcskei},~\bfnm{Helmut}\binits{H.}},
  \bauthor{\bsnm{Grohs},~\bfnm{Philipp}\binits{P.}},
  \bauthor{\bsnm{Kutyniok},~\bfnm{Gitta}\binits{G.}} \AND
  \bauthor{\bsnm{Petersen},~\bfnm{Philipp}\binits{P.}}
(\byear{2019}).
\btitle{Optimal approximation with sparsely connected deep neural networks}.
\bjournal{SIAM Journal on Mathematics of Data Science}
\bvolume{1}
\bpages{8--45}.
\end{barticle}
\endbibitem

\bibitem{bradic2021high}
\begin{barticle}[author]
\bauthor{\bsnm{Bradic},~\bfnm{Jelena}\binits{J.}},
  \bauthor{\bsnm{Ji},~\bfnm{Weijie}\binits{W.}} \AND
  \bauthor{\bsnm{Zhang},~\bfnm{Yuqian}\binits{Y.}}
(\byear{2021}).
\btitle{High-dimensional inference for dynamic treatment effects}.
\bjournal{arXiv preprint arXiv:2110.04924}.
\end{barticle}
\endbibitem

\bibitem{buhlmann2011statistics}
\begin{bbook}[author]
\bauthor{\bsnm{B{\"u}hlmann},~\bfnm{Peter}\binits{P.}} \AND \bauthor{\bsnm{Van
  De~Geer},~\bfnm{Sara}\binits{S.}}
(\byear{2011}).
\btitle{Statistics for high-dimensional data: methods, theory and
  applications}.
\bpublisher{Springer Science \& Business Media}.
\end{bbook}
\endbibitem

\bibitem{chatterjee2011strong}
\begin{barticle}[author]
\bauthor{\bsnm{Chatterjee},~\bfnm{A}\binits{A.}} \AND
  \bauthor{\bsnm{Lahiri},~\bfnm{SN}\binits{S.}}
(\byear{2011}).
\btitle{Strong consistency of lasso estimators}.
\bjournal{Sankhya A}
\bvolume{73}
\bpages{55--78}.
\end{barticle}
\endbibitem

\bibitem{chatterjee2013assumptionless}
\begin{barticle}[author]
\bauthor{\bsnm{Chatterjee},~\bfnm{Sourav}\binits{S.}}
(\byear{2013}).
\btitle{Assumptionless consistency of the lasso}.
\bjournal{arXiv preprint arXiv:1303.5817}.
\end{barticle}
\endbibitem

\bibitem{chernozhukov2021distributional}
\begin{barticle}[author]
\bauthor{\bsnm{Chernozhukov},~\bfnm{Victor}\binits{V.}},
  \bauthor{\bsnm{W{\"u}thrich},~\bfnm{Kaspar}\binits{K.}} \AND
  \bauthor{\bsnm{Zhu},~\bfnm{Yinchu}\binits{Y.}}
(\byear{2021}).
\btitle{Distributional conformal prediction}.
\bjournal{Proceedings of the National Academy of Sciences}
\bvolume{118}
\bpages{e2107794118}.
\end{barticle}
\endbibitem

\bibitem{efron1981jackknife}
\begin{barticle}[author]
\bauthor{\bsnm{Efron},~\bfnm{Bradley}\binits{B.}} \AND
  \bauthor{\bsnm{Stein},~\bfnm{Charles}\binits{C.}}
(\byear{1981}).
\btitle{The jackknife estimate of variance}.
\bjournal{The Annals of Statistics}
\bvolume{9}
\bpages{586--596}.
\end{barticle}
\endbibitem

\bibitem{fei2021estimation}
\begin{barticle}[author]
\bauthor{\bsnm{Fei},~\bfnm{Zhe}\binits{Z.}} \AND
  \bauthor{\bsnm{Li},~\bfnm{Yi}\binits{Y.}}
(\byear{2021}).
\btitle{Estimation and inference for high dimensional generalized linear
  models: A splitting and smoothing approach}.
\bjournal{Journal of Machine Learning Research}
\bvolume{22}
\bpages{1--32}.
\end{barticle}
\endbibitem

\bibitem{fei2021inference}
\begin{barticle}[author]
\bauthor{\bsnm{Fei},~\bfnm{Zhe}\binits{Z.}},
  \bauthor{\bsnm{Zheng},~\bfnm{Qi}\binits{Q.}},
  \bauthor{\bsnm{Hong},~\bfnm{Hyokyoung~G}\binits{H.~G.}} \AND
  \bauthor{\bsnm{Li},~\bfnm{Yi}\binits{Y.}}
(\byear{2023}).
\btitle{Inference for High-Dimensional Censored Quantile Regression}.
\bjournal{Journal of the American Statistical Association}
\bvolume{118}
\bpages{898--912}.
\bnote{PMCID: not yet available}.
\end{barticle}
\endbibitem

\bibitem{fei2019drawing}
\begin{barticle}[author]
\bauthor{\bsnm{Fei},~\bfnm{Zhe}\binits{Z.}},
  \bauthor{\bsnm{Zhu},~\bfnm{Ji}\binits{J.}},
  \bauthor{\bsnm{Banerjee},~\bfnm{Moulinath}\binits{M.}} \AND
  \bauthor{\bsnm{Li},~\bfnm{Yi}\binits{Y.}}
(\byear{2019}).
\btitle{Drawing inferences for high-dimensional linear models: A
  selection-assisted partial regression and smoothing approach}.
\bjournal{Biometrics}
\bvolume{75}
\bpages{551--561}.
\end{barticle}
\endbibitem

\bibitem{frees1989infinite}
\begin{barticle}[author]
\bauthor{\bsnm{Frees},~\bfnm{Edward~W}\binits{E.~W.}}
(\byear{1989}).
\btitle{Infinite order U-statistics}.
\bjournal{Scandinavian Journal of Statistics}
\bvolume{16}
\bpages{29--45}.
\end{barticle}
\endbibitem

\bibitem{fu2000asymptotics}
\begin{barticle}[author]
\bauthor{\bsnm{Fu},~\bfnm{Wenjiang}\binits{W.}} \AND
  \bauthor{\bsnm{Knight},~\bfnm{Keith}\binits{K.}}
(\byear{2000}).
\btitle{Asymptotics for lasso-type estimators}.
\bjournal{The Annals of statistics}
\bvolume{28}
\bpages{1356--1378}.
\end{barticle}
\endbibitem

\bibitem{guo2021inference}
\begin{barticle}[author]
\bauthor{\bsnm{Guo},~\bfnm{Zijian}\binits{Z.}},
  \bauthor{\bsnm{Rakshit},~\bfnm{Prabrisha}\binits{P.}},
  \bauthor{\bsnm{Herman},~\bfnm{Daniel~S}\binits{D.~S.}} \AND
  \bauthor{\bsnm{Chen},~\bfnm{Jinbo}\binits{J.}}
(\byear{2021}).
\btitle{Inference for the case probability in high-dimensional logistic
  regression}.
\bjournal{The Journal of Machine Learning Research}
\bvolume{22}
\bpages{11480--11533}.
\end{barticle}
\endbibitem

\bibitem{hajek1968asymptotic}
\begin{barticle}[author]
\bauthor{\bsnm{H{\'a}jek},~\bfnm{Jaroslav}\binits{J.}}
(\byear{1968}).
\btitle{Asymptotic normality of simple linear rank statistics under
  alternatives}.
\bjournal{The Annals of Mathematical Statistics}
\bvolume{39}
\bpages{325--346}.
\end{barticle}
\endbibitem

\bibitem{hoeffding1992class}
\begin{bincollection}[author]
\bauthor{\bsnm{Hoeffding},~\bfnm{Wassily}\binits{W.}}
(\byear{1992}).
\btitle{A class of statistics with asymptotically normal distribution}.
In \bbooktitle{Breakthroughs in statistics}
\bpages{308--334}.
\bpublisher{Springer}.
\end{bincollection}
\endbibitem

\bibitem{hornik1989multilayer}
\begin{barticle}[author]
\bauthor{\bsnm{Hornik},~\bfnm{Kurt}\binits{K.}},
  \bauthor{\bsnm{Stinchcombe},~\bfnm{Maxwell}\binits{M.}} \AND
  \bauthor{\bsnm{White},~\bfnm{Halbert}\binits{H.}}
(\byear{1989}).
\btitle{Multilayer feedforward networks are universal approximators}.
\bjournal{Neural networks}
\bvolume{2}
\bpages{359--366}.
\end{barticle}
\endbibitem

\bibitem{horvath2013dna}
\begin{barticle}[author]
\bauthor{\bsnm{Horvath},~\bfnm{Steve}\binits{S.}}
(\byear{2013}).
\btitle{DNA methylation age of human tissues and cell types}.
\bjournal{Genome biology}
\bvolume{14}
\bpages{1--20}.
\end{barticle}
\endbibitem

\bibitem{janson1984asymptotic}
\begin{barticle}[author]
\bauthor{\bsnm{Janson},~\bfnm{Svante}\binits{S.}}
(\byear{1984}).
\btitle{The asymptotic distributions of incomplete U-statistics}.
\bjournal{Zeitschrift f{\"u}r Wahrscheinlichkeitstheorie und Verwandte Gebiete}
\bvolume{66}
\bpages{495--505}.
\end{barticle}
\endbibitem

\bibitem{javanmard2014confidence}
\begin{barticle}[author]
\bauthor{\bsnm{Javanmard},~\bfnm{Adel}\binits{A.}} \AND
  \bauthor{\bsnm{Montanari},~\bfnm{Andrea}\binits{A.}}
(\byear{2014}).
\btitle{Confidence intervals and hypothesis testing for high-dimensional
  regression}.
\bjournal{Journal of Machine Learning Research}
\bvolume{15}
\bpages{2869--2909}.
\end{barticle}
\endbibitem

\bibitem{kato2023review}
\begin{barticle}[author]
\bauthor{\bsnm{Kato},~\bfnm{Yuko}\binits{Y.}},
  \bauthor{\bsnm{Tax},~\bfnm{David~MJ}\binits{D.~M.}} \AND
  \bauthor{\bsnm{Loog},~\bfnm{Marco}\binits{M.}}
(\byear{2023}).
\btitle{A review of nonconformity measures for conformal prediction in
  regression}.
\bjournal{Conformal and Probabilistic Prediction with Applications}
\bpages{369--383}.
\end{barticle}
\endbibitem

\bibitem{khosravi2010lower}
\begin{barticle}[author]
\bauthor{\bsnm{Khosravi},~\bfnm{Abbas}\binits{A.}},
  \bauthor{\bsnm{Nahavandi},~\bfnm{Saeid}\binits{S.}},
  \bauthor{\bsnm{Creighton},~\bfnm{Doug}\binits{D.}} \AND
  \bauthor{\bsnm{Atiya},~\bfnm{Amir~F}\binits{A.~F.}}
(\byear{2010}).
\btitle{Lower upper bound estimation method for construction of neural
  network-based prediction intervals}.
\bjournal{IEEE transactions on neural networks}
\bvolume{22}
\bpages{337--346}.
\end{barticle}
\endbibitem

\bibitem{kim2020predictive}
\begin{barticle}[author]
\bauthor{\bsnm{Kim},~\bfnm{Byol}\binits{B.}},
  \bauthor{\bsnm{Xu},~\bfnm{Chen}\binits{C.}} \AND
  \bauthor{\bsnm{Barber},~\bfnm{Rina}\binits{R.}}
(\byear{2020}).
\btitle{Predictive inference is free with the jackknife+-after-bootstrap}.
\bjournal{Advances in Neural Information Processing Systems}
\bvolume{33}
\bpages{4138--4149}.
\end{barticle}
\endbibitem

\bibitem{kohler2005adaptive}
\begin{barticle}[author]
\bauthor{\bsnm{Kohler},~\bfnm{Michael}\binits{M.}} \AND \bauthor{\bsnm{Krzy{\.
  z}ak},~\bfnm{Adam}\binits{A.}}
(\byear{2005}).
\btitle{Adaptive regression estimation with multilayer feedforward neural
  networks}.
\bjournal{Nonparametric Statistics}
\bvolume{17}
\bpages{891--913}.
\end{barticle}
\endbibitem

\bibitem{kohler2021rate}
\begin{barticle}[author]
\bauthor{\bsnm{Kohler},~\bfnm{Michael}\binits{M.}} \AND
  \bauthor{\bsnm{Langer},~\bfnm{Sophie}\binits{S.}}
(\byear{2021}).
\btitle{On the rate of convergence of fully connected deep neural network
  regression estimates}.
\bjournal{The Annals of Statistics}
\bvolume{49}
\bpages{2231--2249}.
\end{barticle}
\endbibitem

\bibitem{kong2021high}
\begin{barticle}[author]
\bauthor{\bsnm{Kong},~\bfnm{Shengchun}\binits{S.}},
  \bauthor{\bsnm{Yu},~\bfnm{Zhuqing}\binits{Z.}},
  \bauthor{\bsnm{Zhang},~\bfnm{Xianyang}\binits{X.}} \AND
  \bauthor{\bsnm{Cheng},~\bfnm{Guang}\binits{G.}}
(\byear{2021}).
\btitle{High Dimensional Robust Inference for Cox Regression Models using
  De-sparsified Lasso}.
\bjournal{Scandinavian Journal of Statistics}
\bvolume{48}
\bpages{1068--1095}.
\bnote{doi: 10.1111/sjos.12543}.
\end{barticle}
\endbibitem

\bibitem{kaggle2023life}
\begin{bmisc}[author]
\bauthor{\bsnm{Kumarajarshi}}
(\byear{2023}).
\btitle{Life Expectancy (WHO)}.
\bnote{Kaggle dataset based on WHO data}.
\end{bmisc}
\endbibitem

\bibitem{lee2019u}
\begin{bbook}[author]
\bauthor{\bsnm{Lee},~\bfnm{A~J}\binits{A.~J.}}
(\byear{2019}).
\btitle{U-statistics: Theory and Practice}.
\bpublisher{Routledge}.
\end{bbook}
\endbibitem

\bibitem{lee2016exact}
\begin{barticle}[author]
\bauthor{\bsnm{Lee},~\bfnm{Jason~D}\binits{J.~D.}},
  \bauthor{\bsnm{Sun},~\bfnm{Dennis~L}\binits{D.~L.}},
  \bauthor{\bsnm{Sun},~\bfnm{Yuekai}\binits{Y.}} \AND
  \bauthor{\bsnm{Taylor},~\bfnm{Jonathan~E}\binits{J.~E.}}
(\byear{2016}).
\btitle{Exact post-selection inference, with application to the lasso}.
\bjournal{The Annals of Statistics}
\bvolume{44}
\bpages{907--927}.
\end{barticle}
\endbibitem

\bibitem{lei2018distribution}
\begin{barticle}[author]
\bauthor{\bsnm{Lei},~\bfnm{Jing}\binits{J.}},
  \bauthor{\bsnm{G’Sell},~\bfnm{Max}\binits{M.}},
  \bauthor{\bsnm{Rinaldo},~\bfnm{Alessandro}\binits{A.}},
  \bauthor{\bsnm{Tibshirani},~\bfnm{Ryan~J}\binits{R.~J.}} \AND
  \bauthor{\bsnm{Wasserman},~\bfnm{Larry}\binits{L.}}
(\byear{2018}).
\btitle{Distribution-free predictive inference for regression}.
\bjournal{Journal of the American Statistical Association}
\bvolume{113}
\bpages{1094--1111}.
\end{barticle}
\endbibitem

\bibitem{lin2016dna}
\begin{barticle}[author]
\bauthor{\bsnm{Lin},~\bfnm{Qiong}\binits{Q.}},
  \bauthor{\bsnm{Weidner},~\bfnm{Carola~I}\binits{C.~I.}},
  \bauthor{\bsnm{Costa},~\bfnm{Ivan~G}\binits{I.~G.}},
  \bauthor{\bsnm{Marioni},~\bfnm{Riccardo~E}\binits{R.~E.}},
  \bauthor{\bsnm{Ferreira},~\bfnm{Marcelo~RP}\binits{M.~R.}},
  \bauthor{\bsnm{Deary},~\bfnm{Ian~J}\binits{I.~J.}} \AND
  \bauthor{\bsnm{Wagner},~\bfnm{Wolfgang}\binits{W.}}
(\byear{2016}).
\btitle{DNA methylation levels at individual age-associated CpG sites can be
  indicative for life expectancy}.
\bjournal{Aging (Albany NY)}
\bvolume{8}
\bpages{394}.
\end{barticle}
\endbibitem

\bibitem{lu2019dna}
\begin{barticle}[author]
\bauthor{\bsnm{Lu},~\bfnm{Ake~T}\binits{A.~T.}},
  \bauthor{\bsnm{Quach},~\bfnm{Austin}\binits{A.}},
  \bauthor{\bsnm{Wilson},~\bfnm{James~G}\binits{J.~G.}},
  \bauthor{\bsnm{Reiner},~\bfnm{Alex~P}\binits{A.~P.}},
  \bauthor{\bsnm{Aviv},~\bfnm{Abraham}\binits{A.}},
  \bauthor{\bsnm{Raj},~\bfnm{Kenneth}\binits{K.}},
  \bauthor{\bsnm{Hou},~\bfnm{Lifang}\binits{L.}},
  \bauthor{\bsnm{Baccarelli},~\bfnm{Andrea~A}\binits{A.~A.}},
  \bauthor{\bsnm{Li},~\bfnm{Yun}\binits{Y.}},
  \bauthor{\bsnm{Stewart},~\bfnm{James~D}\binits{J.~D.}} \betal{et~al.}
(\byear{2019}).
\btitle{DNA methylation GrimAge strongly predicts lifespan and healthspan}.
\bjournal{Aging (albany NY)}
\bvolume{11}
\bpages{303}.
\end{barticle}
\endbibitem

\bibitem{marioni2015dna}
\begin{barticle}[author]
\bauthor{\bsnm{Marioni},~\bfnm{Riccardo~E}\binits{R.~E.}},
  \bauthor{\bsnm{Shah},~\bfnm{Sonia}\binits{S.}},
  \bauthor{\bsnm{McRae},~\bfnm{Allan~F}\binits{A.~F.}},
  \bauthor{\bsnm{Chen},~\bfnm{Brian~H}\binits{B.~H.}},
  \bauthor{\bsnm{Colicino},~\bfnm{Elena}\binits{E.}},
  \bauthor{\bsnm{Harris},~\bfnm{Sarah~E}\binits{S.~E.}},
  \bauthor{\bsnm{Gibson},~\bfnm{Jude}\binits{J.}},
  \bauthor{\bsnm{Henders},~\bfnm{Anjali~K}\binits{A.~K.}},
  \bauthor{\bsnm{Redmond},~\bfnm{Paul}\binits{P.}},
  \bauthor{\bsnm{Cox},~\bfnm{Simon~R}\binits{S.~R.}} \betal{et~al.}
(\byear{2015}).
\btitle{DNA methylation age of blood predicts all-cause mortality in later
  life}.
\bjournal{Genome biology}
\bvolume{16}
\bpages{1--12}.
\end{barticle}
\endbibitem

\bibitem{mccaffrey1994convergence}
\begin{barticle}[author]
\bauthor{\bsnm{McCaffrey},~\bfnm{Daniel~F}\binits{D.~F.}} \AND
  \bauthor{\bsnm{Gallant},~\bfnm{A~Ronald}\binits{A.~R.}}
(\byear{1994}).
\btitle{Convergence rates for single hidden layer feedforward networks}.
\bjournal{Neural Networks}
\bvolume{7}
\bpages{147--158}.
\end{barticle}
\endbibitem

\bibitem{mccartney2018investigating}
\begin{barticle}[author]
\bauthor{\bsnm{McCartney},~\bfnm{Daniel~L}\binits{D.~L.}},
  \bauthor{\bsnm{Stevenson},~\bfnm{Anna~J}\binits{A.~J.}},
  \bauthor{\bsnm{Walker},~\bfnm{Rosie~M}\binits{R.~M.}},
  \bauthor{\bsnm{Gibson},~\bfnm{Jude}\binits{J.}},
  \bauthor{\bsnm{Morris},~\bfnm{Stewart~W}\binits{S.~W.}},
  \bauthor{\bsnm{Campbell},~\bfnm{Archie}\binits{A.}},
  \bauthor{\bsnm{Murray},~\bfnm{Alison~D}\binits{A.~D.}},
  \bauthor{\bsnm{Whalley},~\bfnm{Heather~C}\binits{H.~C.}},
  \bauthor{\bsnm{Porteous},~\bfnm{David~J}\binits{D.~J.}},
  \bauthor{\bsnm{McIntosh},~\bfnm{Andrew~M}\binits{A.~M.}} \betal{et~al.}
(\byear{2018}).
\btitle{Investigating the relationship between DNA methylation age acceleration
  and risk factors for Alzheimer's disease}.
\bjournal{Alzheimer's \& Dementia: Diagnosis, Assessment \& Disease Monitoring}
\bvolume{10}
\bpages{429--437}.
\end{barticle}
\endbibitem

\bibitem{mentch2016quantifying}
\begin{barticle}[author]
\bauthor{\bsnm{Mentch},~\bfnm{Lucas}\binits{L.}} \AND
  \bauthor{\bsnm{Hooker},~\bfnm{Giles}\binits{G.}}
(\byear{2016}).
\btitle{Quantifying uncertainty in random forests via confidence intervals and
  hypothesis tests}.
\bjournal{The Journal of Machine Learning Research}
\bvolume{17}
\bpages{841--881}.
\end{barticle}
\endbibitem

\bibitem{messoudi2020conformal}
\begin{binproceedings}[author]
\bauthor{\bsnm{Messoudi},~\bfnm{Soundouss}\binits{S.}},
  \bauthor{\bsnm{Destercke},~\bfnm{S{\'e}bastien}\binits{S.}} \AND
  \bauthor{\bsnm{Rousseau},~\bfnm{Sylvain}\binits{S.}}
(\byear{2020}).
\btitle{Conformal multi-target regression using neural networks}.
In \bbooktitle{Conformal and Probabilistic Prediction and Applications}
\bpages{65--83}.
\bpublisher{PMLR}.
\end{binproceedings}
\endbibitem

\bibitem{milicic2023utility}
\begin{barticle}[author]
\bauthor{\bsnm{Milicic},~\bfnm{Lidija}\binits{L.}},
  \bauthor{\bsnm{Porter},~\bfnm{Tenielle}\binits{T.}},
  \bauthor{\bsnm{Vacher},~\bfnm{Michael}\binits{M.}} \AND
  \bauthor{\bsnm{Laws},~\bfnm{Simon~M}\binits{S.~M.}}
(\byear{2023}).
\btitle{Utility of DNA Methylation as a Biomarker in Aging and Alzheimer’s
  Disease}.
\bjournal{Journal of Alzheimer's Disease Reports}
\bvolume{7}
\bpages{475--503}.
\end{barticle}
\endbibitem

\bibitem{ning2017general}
\begin{barticle}[author]
\bauthor{\bsnm{Ning},~\bfnm{Yang}\binits{Y.}} \AND
  \bauthor{\bsnm{Liu},~\bfnm{Han}\binits{H.}}
(\byear{2017}).
\btitle{A general theory of hypothesis tests and confidence regions for sparse
  high dimensional models}.
\bjournal{The Annals of Statistics}
\bvolume{45}
\bpages{158--195}.
\end{barticle}
\endbibitem

\bibitem{papadopoulos2008inductive}
\begin{bincollection}[author]
\bauthor{\bsnm{Papadopoulos},~\bfnm{Harris}\binits{H.}}
(\byear{2008}).
\btitle{Inductive conformal prediction: Theory and application to neural
  networks}.
In \bbooktitle{Tools in artificial intelligence}
\bpublisher{Citeseer}.
\end{bincollection}
\endbibitem

\bibitem{papadopoulos2007conformal}
\begin{binproceedings}[author]
\bauthor{\bsnm{Papadopoulos},~\bfnm{Harris}\binits{H.}},
  \bauthor{\bsnm{Vovk},~\bfnm{Volodya}\binits{V.}} \AND
  \bauthor{\bsnm{Gammerman},~\bfnm{Alex}\binits{A.}}
(\byear{2007}).
\btitle{Conformal prediction with neural networks}.
In \bbooktitle{19th IEEE International Conference on Tools with Artificial
  Intelligence (ICTAI 2007)}
\bvolume{2}
\bpages{388--395}.
\bpublisher{IEEE}.
\end{binproceedings}
\endbibitem

\bibitem{rigollet2011exponential}
\begin{barticle}[author]
\bauthor{\bsnm{Rigollet},~\bfnm{Philippe}\binits{P.}} \AND
  \bauthor{\bsnm{Tsybakov},~\bfnm{Alexandre}\binits{A.}}
(\byear{2011}).
\btitle{Exponential screening and optimal rates of sparse estimation}.
\bjournal{The Annals of Statistics}
\bvolume{39}
\bpages{731--771}.
\end{barticle}
\endbibitem

\bibitem{romano2020classification}
\begin{barticle}[author]
\bauthor{\bsnm{Romano},~\bfnm{Yaniv}\binits{Y.}},
  \bauthor{\bsnm{Sesia},~\bfnm{Matteo}\binits{M.}} \AND
  \bauthor{\bsnm{Candes},~\bfnm{Emmanuel}\binits{E.}}
(\byear{2020}).
\btitle{Classification with valid and adaptive coverage}.
\bjournal{Advances in Neural Information Processing Systems}
\bvolume{33}
\bpages{3581--3591}.
\end{barticle}
\endbibitem

\bibitem{schmidt2020nonparametric}
\begin{barticle}[author]
\bauthor{\bsnm{Schmidt-Hieber},~\bfnm{Johannes}\binits{J.}}
(\byear{2020}).
\btitle{Nonparametric regression using deep neural networks with ReLU
  activation function}.
\bjournal{The Annals of Statistics}
\bvolume{48}
\bpages{1875--1897}.
\end{barticle}
\endbibitem

\bibitem{schupbach2020quantifying}
\begin{binproceedings}[author]
\bauthor{\bsnm{Schupbach},~\bfnm{Jordan}\binits{J.}},
  \bauthor{\bsnm{Sheppard},~\bfnm{John~W}\binits{J.~W.}} \AND
  \bauthor{\bsnm{Forrester},~\bfnm{Tyler}\binits{T.}}
(\byear{2020}).
\btitle{Quantifying uncertainty in neural network ensembles using
  u-statistics}.
In \bbooktitle{2020 International Joint Conference on Neural Networks (IJCNN)}
\bpages{1--8}.
\bpublisher{IEEE}.
\end{binproceedings}
\endbibitem

\bibitem{sugden2022association}
\begin{barticle}[author]
\bauthor{\bsnm{Sugden},~\bfnm{Karen}\binits{K.}},
  \bauthor{\bsnm{Caspi},~\bfnm{Avshalom}\binits{A.}},
  \bauthor{\bsnm{Elliott},~\bfnm{Maxwell~L}\binits{M.~L.}},
  \bauthor{\bsnm{Bourassa},~\bfnm{Kyle~J}\binits{K.~J.}},
  \bauthor{\bsnm{Chamarti},~\bfnm{Kartik}\binits{K.}},
  \bauthor{\bsnm{Corcoran},~\bfnm{David~L}\binits{D.~L.}},
  \bauthor{\bsnm{Hariri},~\bfnm{Ahmad~R}\binits{A.~R.}},
  \bauthor{\bsnm{Houts},~\bfnm{Renate~M}\binits{R.~M.}},
  \bauthor{\bsnm{Kothari},~\bfnm{Meeraj}\binits{M.}},
  \bauthor{\bsnm{Kritchevsky},~\bfnm{Stephen}\binits{S.}} \betal{et~al.}
(\byear{2022}).
\btitle{Association of pace of aging measured by blood-based DNA methylation
  with age-related cognitive impairment and dementia}.
\bjournal{Neurology}
\bvolume{99}
\bpages{e1402--e1413}.
\end{barticle}
\endbibitem

\bibitem{tibshirani2019conformal}
\begin{binproceedings}[author]
\bauthor{\bsnm{Tibshirani},~\bfnm{Ryan~J}\binits{R.~J.}},
  \bauthor{\bsnm{Foygel~Barber},~\bfnm{Rina}\binits{R.}},
  \bauthor{\bsnm{Candes},~\bfnm{Emmanuel}\binits{E.}} \AND
  \bauthor{\bsnm{Ramdas},~\bfnm{Aaditya}\binits{A.}}
(\byear{2019}).
\btitle{Conformal prediction under covariate shift}.
In \bbooktitle{Advances in Neural Information Processing Systems}
(\beditor{\bfnm{H.}\binits{H.}~\bsnm{Wallach}},
  \beditor{\bfnm{H.}\binits{H.}~\bsnm{Larochelle}},
  \beditor{\bfnm{A.}\binits{A.}~\bsnm{Beygelzimer}},
  \beditor{\bfnm{F.}\binits{F.}~\bparticle{d\textquotesingle}
  \bsnm{Alch\'{e}-Buc}}, \beditor{\bfnm{E.}\binits{E.}~\bsnm{Fox}} \AND
  \beditor{\bfnm{R.}\binits{R.}~\bsnm{Garnett}}, eds.)
\bvolume{32}.
\bpublisher{Curran Associates, Inc.}
\end{binproceedings}
\endbibitem

\bibitem{van2014asymptotically}
\begin{barticle}[author]
\bauthor{\bparticle{Van~de} \bsnm{Geer},~\bfnm{Sara}\binits{S.}},
  \bauthor{\bsnm{B{\"u}hlmann},~\bfnm{Peter}\binits{P.}},
  \bauthor{\bsnm{Ritov},~\bfnm{Ya’acov}\binits{Y.}} \AND
  \bauthor{\bsnm{Dezeure},~\bfnm{Ruben}\binits{R.}}
(\byear{2014}).
\btitle{On asymptotically optimal confidence regions and tests for
  high-dimensional models}.
\bjournal{The Annals of Statistics}
\bvolume{42}
\bpages{1166--1202}.
\end{barticle}
\endbibitem

\bibitem{van2000asymptotic}
\begin{bbook}[author]
\bauthor{\bparticle{Van~der} \bsnm{Vaart},~\bfnm{Aad~W}\binits{A.~W.}}
(\byear{2000}).
\btitle{Asymptotic Statistics}
\bvolume{3}.
\bpublisher{Cambridge University Press}.
\end{bbook}
\endbibitem

\bibitem{wager2018estimation}
\begin{barticle}[author]
\bauthor{\bsnm{Wager},~\bfnm{Stefan}\binits{S.}} \AND
  \bauthor{\bsnm{Athey},~\bfnm{Susan}\binits{S.}}
(\byear{2018}).
\btitle{Estimation and inference of heterogeneous treatment effects using
  random forests}.
\bjournal{Journal of the American Statistical Association}
\bvolume{113}
\bpages{1228--1242}.
\end{barticle}
\endbibitem

\bibitem{wainwright2019high}
\begin{bbook}[author]
\bauthor{\bsnm{Wainwright},~\bfnm{Martin~J}\binits{M.~J.}}
(\byear{2019}).
\btitle{High-dimensional statistics: A non-asymptotic viewpoint}
\bvolume{48}.
\bpublisher{Cambridge university press}.
\end{bbook}
\endbibitem

\bibitem{wang2022quantifying}
\begin{barticle}[author]
\bauthor{\bsnm{Wang},~\bfnm{Qing}\binits{Q.}} \AND
  \bauthor{\bsnm{Wei},~\bfnm{Yujie}\binits{Y.}}
(\byear{2022}).
\btitle{Quantifying uncertainty of subsampling-based ensemble methods under a
  U-statistic framework}.
\bjournal{Journal of Statistical Computation and Simulation}
\bvolume{92}
\bpages{3706--3726}.
\end{barticle}
\endbibitem

\bibitem{yu2020epigenetic}
\begin{barticle}[author]
\bauthor{\bsnm{Yu},~\bfnm{Ming}\binits{M.}},
  \bauthor{\bsnm{Hazelton},~\bfnm{William~D}\binits{W.~D.}},
  \bauthor{\bsnm{Luebeck},~\bfnm{Georg~E}\binits{G.~E.}} \AND
  \bauthor{\bsnm{Grady},~\bfnm{William~M}\binits{W.~M.}}
(\byear{2020}).
\btitle{Epigenetic Aging: More Than Just a Clock When It Comes to Cancer}.
\bjournal{Cancer research}
\bvolume{80}
\bpages{367--374}.
\end{barticle}
\endbibitem

\bibitem{zavatone2021asymptotics}
\begin{barticle}[author]
\bauthor{\bsnm{Zavatone-Veth},~\bfnm{Jacob}\binits{J.}},
  \bauthor{\bsnm{Canatar},~\bfnm{Abdulkadir}\binits{A.}},
  \bauthor{\bsnm{Ruben},~\bfnm{Ben}\binits{B.}} \AND
  \bauthor{\bsnm{Pehlevan},~\bfnm{Cengiz}\binits{C.}}
(\byear{2021}).
\btitle{Asymptotics of representation learning in finite Bayesian neural
  networks}.
\bjournal{Advances in neural information processing systems}
\bvolume{34}
\bpages{24765--24777}.
\end{barticle}
\endbibitem

\bibitem{zhang2014confidence}
\begin{barticle}[author]
\bauthor{\bsnm{Zhang},~\bfnm{Cun-Hui}\binits{C.-H.}} \AND
  \bauthor{\bsnm{Zhang},~\bfnm{Stephanie~S}\binits{S.~S.}}
(\byear{2014}).
\btitle{Confidence intervals for low dimensional parameters in high dimensional
  linear models}.
\bjournal{Journal of the Royal Statistical Society: Series B (Statistical
  Methodology)}
\bvolume{76}
\bpages{217--242}.
\end{barticle}
\endbibitem

\bibitem{zhang2023bootstrap}
\begin{barticle}[author]
\bauthor{\bsnm{Zhang},~\bfnm{Yunyi}\binits{Y.}} \AND
  \bauthor{\bsnm{Politis},~\bfnm{Dimitris~N}\binits{D.~N.}}
(\byear{2023}).
\btitle{Bootstrap prediction intervals with asymptotic conditional validity and
  unconditional guarantees}.
\bjournal{Information and Inference: A Journal of the IMA}
\bvolume{12}
\bpages{157--209}.
\end{barticle}
\endbibitem

\bibitem{zhao2006model}
\begin{barticle}[author]
\bauthor{\bsnm{Zhao},~\bfnm{Peng}\binits{P.}} \AND
  \bauthor{\bsnm{Yu},~\bfnm{Bin}\binits{B.}}
(\byear{2006}).
\btitle{On model selection consistency of Lasso}.
\bjournal{Journal of Machine Learning Research}
\bvolume{7}
\bpages{2541--2563}.
\end{barticle}
\endbibitem

\bibitem{zhu2018linear}
\begin{barticle}[author]
\bauthor{\bsnm{Zhu},~\bfnm{Yinchu}\binits{Y.}} \AND
  \bauthor{\bsnm{Bradic},~\bfnm{Jelena}\binits{J.}}
(\byear{2018}).
\btitle{Linear hypothesis testing in dense high-dimensional linear models}.
\bjournal{Journal of the American Statistical Association}
\bvolume{113}
\bpages{1583--1600}.
\end{barticle}
\endbibitem

\end{thebibliography}




\end{document}